\newif\ifarXiv         
\newif\ifjournal        
  \journalfalse       \arXivtrue      

\ifjournal  \input{macro_journal}   \fi

\documentclass[times,11pt]{article} 

\usepackage[small,compact]{titlesec}
\usepackage{amsmath,amssymb,amsfonts,mathabx,setspace}
\usepackage{graphicx,caption,epsfig,subcaption,epsfig,wrapfig,sidecap}
\usepackage{url,color,verbatim,algorithmicx}
\usepackage[sort,compress]{cite}
\usepackage{soul}
\usepackage{xspace}

\usepackage[noend]{algpseudocode}
\usepackage[ruled,boxed]{algorithm}

\newcounter{multifig} 

\newcommand{\figcaption}[1]
{\stepcounter{multifig}
\addcontentsline{lof}{figure}{\string\numberline {\arabic{multifig}}{\ignorespaces #1}} Figure \arabic{multifig}: #1}

\usepackage{multirow}

\usepackage[scaled]{helvet}
\usepackage[T1]{fontenc}
\usepackage[bookmarks=true, bookmarksnumbered=true, colorlinks=true,   pdfstartview=FitV,
linkcolor=blue, citecolor=blue, urlcolor=blue]{hyperref}
\usepackage{multirow, makecell, colortbl, hhline}
\usepackage{booktabs} 
\usepackage{enumitem} 
\usepackage{etoc}          
\usepackage[dvipsnames]{xcolor}

\usepackage{fullpage}
\newtheorem{theorem}{Theorem}

\newtheorem{definition}[theorem]{Definition}
\newtheorem{myexample}[theorem]{Example}

\newtheorem{lemma}[theorem]{Lemma}

\newenvironment{proof}[1][Proof]{\noindent\textbf{#1.} }{\ \rule{0.5em}{0.5em}}

\numberwithin{equation}{section}
\numberwithin{theorem}{section}

\usepackage{enumitem}
\usepackage{xcolor}

\def\R{\mathbb{R}}  
\def\prob{\mathbb{P}} \def\E{\mathbb{E}}

\def\bx{\mathbf{x}}

\def\bX{\mathbf{X}}

\newcommand{\argmax}[1]{\underset{#1}{\operatorname{arg}\operatorname{max}}\;}

\newcommand{\floor}[1]{\lfloor{#1}\rfloor}

\definecolor{mygrey}{gray}{0.75}

\def\XXint#1#2#3{{\setbox0=\hbox{$#1{#2#3}{\int}$ }
		\vcenter{\hbox{$#2#3$ }}\kern-.6\wd0}}

\ifarXiv
\title{Benchmarking optimality of time series classification methods\\ in distinguishing diffusions
}

\usepackage{authblk} 
\author[1]{Zehong Zhang}
\author[1]{Fei Lu}
\author[2,4]{Esther Xu Fei} 
 \author[3]{Terry Lyons}
\author[4,5]{\\Yannis Kevrekidis}
\author[6]{Tom Woolf}
\affil[1]{\footnotesize  Department of Mathematics, Johns Hopkins University, Baltimore, MD 21218, USA
 \href{feilu@math.jhu.edu} {feilu@math.jhu.edu}
}
\affil[2]{\footnotesize  Department of Environmental Health and Engineering, Johns Hopkins University, Baltimore, MD 21218, USA 
}
\affil[3]{\footnotesize  Mathematical Institute, University of Oxford, Oxford, United Kingdom 
} 
\affil[4]{\footnotesize  Department of Applied Mathematics and Mathematics, Johns Hopkins University, Baltimore, MD 21218, USA
}
\affil[5]{\footnotesize Departments of Chemical and Biomolecular Engineering,Johns Hopkins University, Baltimore, MD 21218, USA}

\affil[6]{\footnotesize School of Medicine, Johns Hopkins University, Baltimore, MD 21218, USA.
}
\date{}
\fi

\begin{document}
		\maketitle

	\ifarXiv
		\begin{abstract}
		\textbf{Abstract} 
		Statistical optimality benchmarking is crucial for analyzing and designing time series classification (TSC) algorithms.  
This study proposes to benchmark the optimality of TSC algorithms in distinguishing diffusion processes by the likelihood ratio test (LRT). 
The LRT is an optimal classifier by the Neyman-Pearson lemma. The LRT benchmarks are computationally efficient because the LRT does not need training, and the diffusion processes can be efficiently simulated and are flexible to reflect the specific features of real-world applications. We demonstrate the benchmarking with three widely-used TSC algorithms: random forest, ResNet, and ROCKET. These algorithms can achieve the LRT optimality for univariate time series and multivariate Gaussian processes. However, these model-agnostic algorithms are suboptimal in classifying high-dimensional nonlinear multivariate time series. 
Additionally, the LRT benchmark provides tools to analyze the dependence of classification accuracy on the time length, dimension, temporal sampling frequency, and randomness of the time series.  
\end{abstract}
	\noindent{\bf Key words}  Times series classification, Likelihood ratio test, Optimal benchmark, Stochastic differential equations, ResNet, ROCKET, Random forest
	\vspace{3mm}
	\fi
	

\section{Introduction}
Time series classification (TSC) is one of the central tasks in time series analysis and streaming data processing. Recent years have seen an explosion in the collection of time series data and a surge of TSC algorithms (see e.g.,\cite{breiman2001random,lyons2014_RoughPaths,lines2016hive,wang2017time,bagnall2017great,bagnall2018uea,ismail2019deep,dempster2020rocket,ruiz2021great,ismail2020inceptiontime,morrill2021_GeneralisedSignature}). In particular, the recent reviews \cite{bagnall2017great,ismail2019deep,ruiz2021great} have thoroughly compared dozens of TSC algorithms on hundreds of public bakeoff datasets, providing a valuable understanding of the algorithms and the TSC tasks.   
	
However, an optimality benchmark remains missing. In this study, we focus on statistical optimality in terms of \emph{Neyman-Pearson lemma} \cite{neyman1933ix}: an \emph{optimal classifier} has the lowest false positive rate among classifiers with a controlled level of false negative rate. The need for such an optimality benchmark grows along with the fast-growing numbers of datasets and algorithms. Due to a lack of statistical models for the bakeoff datasets, current empirical benchmarks compare all methods using bakeoff datasets and select the top performer in terms of empirical accuracy and efficiency. Yet, the top performers are not cleared to achieve statistical optimality, and their ranks can vary with datasets. Thus, it is important to design a controlled statistical optimality benchmark.

We propose to benchmark the optimality of \emph{binary} TSC algorithms in \emph{distinguishing diffusion processes} by the \emph{likelihood ratio test} (LRT). Our benchmark has three appealing properties: providing an optimal reference, being flexible in design, and being computationally scalable. (i) It has a theory-guaranteed optimal reference because the LRT is an optimal classifier by the Neyman-Pearson lemma \cite{neyman1933ix} (in statistical terms, it is uniformly most powerful). Thus, a TSC method reaching the benchmark is guaranteed optimal for classifying the underlying stochastic process, not sensitive to sampled datasets. (ii) It is flexible in design to reflect the properties of various time series data in applications. Because diffusion processes from stochastic differential equations (SDEs) provide a large variety of Markov processes whose likelihood can be computed. These processes are flexible to reflect the specific features of real-world applications \cite{pavliotis2014stochastic,oksendal2013_sde}, ranging from univariate to multivariate time series, from Gaussian processes to highly nonlinear non-Gaussian processes, and from small to large randomness. (ii) The benchmarking test is computationally efficient and scalable.  The LRT does not require training and has a negligible computation cost. Also, the SDEs can systematically generate large datasets with different lengths, nonlinearities, and levels of randomness. Therefore, our benchmark of LRT for diffusion processes provides a reference of optimality for the performance (such as the ROC curves and accuracy) of all TSC algorithms. 

We demonstrate the LRT benchmarking using three widely-used TSC algorithms: random forest \cite{breiman2001random}, ROCKET \cite{dempster2020rocket}, and ResNet \cite{wang2017time}, in five representative classes of diffusion processes. The five processes are Brownian motions with constant drifts, 1-dimensional nonlinear diffusions with different potentials \cite{oksendal2013_sde}, 1-dimensional linear and nonlinear diffusions \cite{morrill2021_GeneralisedSignature}, multivariate Ornstein-Uhlenbeck processes \cite{pavliotis2014stochastic}, and high-dimensional interacting particle systems \cite{MT14,LZTM19pnas}. Test results show that the three algorithms achieve optimality in the case of Brownian motions with constant drifts, and they are near optimal for the nonlinear univariate time series and multivariate Gaussian processes. However, these three model-agnostic algorithms are significantly less accurate than the model-aware LRT in the case of high-dimensional nonlinear non-gaussian processes. Thus, it would be helpful to incorporate model information in developing next-generation TSC algorithms.

Additionally, the LRT benchmarks show that the optimal accuracy of TSC depends on the time series's length, dimension, and temporal sampling frequency. Analysis and numerical tests show that the accuracy increases with either time length or dimension, which enlarges the effective sample size. However, high-frequency data does not significantly improve the classification rates, since the high-frequency data does not increase the effective sample size.

	
The paper is organized as follows. In Section \ref{sec:tsc_LRT}, we cast the TSC as the learning of a function that maps a time series to a binary output so that a TSC algorithm can be viewed as a hypothesis testing method. In particular, we point out that the likelihood ratio test (LRT) is a uniformly most powerful test by the Neyman-Pearson Lemma. Additionally, we show the computation of the likelihood ratio for diffusion processes. Section \ref{sec:analytical-LRT} analytically computes the LRT for two Gaussian processes. The analysis shows the dependence of the classification accuracy on the time series's dimension, length, and frequency. Section \ref{sec:examples} describes three examples of nonlinear diffusion processes and specifies the data generation for benchmarking tests. These examples showcase the design of benchmarking tests. We present in Section \ref{sec:num} the test results of benchmarking three scalable TSC algorithms: the random forest, ResNet, and ROCKET. Finally, the Appendix briefly reviews the Girsanov theorem and hypothesis testing. 
	
\section{Time series classification and distinguishing diffusions}\label{sec:tsc_LRT}
We recast binary time series classification as a hypothesis testing problem, so that the likelihood ratio test (LRT) provides an optimal classifier by the Neyman-Pearson Lemma. On the other hand, diffusion processes provide a large variety of time series whose LRT can be computed in a scalable fashion. Thus, we propose to benchmark the optimality of TSC classifiers by LRT in distinguishing diffusions. 

\subsection{TSC as a function learning problem}\label{sec:tsc_learning}
In the lens of statistical learning, a binary TSC algorithm learns the probabilities that the time series belongs to two classes from training data \cite{james2013introduction,CuckerSamle02}. 

Let the data be the time series (either univariate or multivariate) and their labels,  
\[
\textbf{Data:  } \quad \{\bx^{(m)}, y^{(m)} \}_{m=1}^M,  \quad \bx^{(m)} \in \R^{d\times (L+1) }, y^{(m)} \in \{0,1\},   
\]
where for each $m$, $\bx^{(m)} = x_{t_0:t_L}^{(m)} = (x_{t_1},\ldots, x_{t_L})^{(m)} $ is a sample path of a stochastic process  $\bX= X_{t_0:t_L}$ with $t_0<t_1< \ldots < t_L$ denoting time indices. Here $y^{(m)}$ has a label $1$ if the times series $\bx^{(m)}$ is in class $\theta_1$; otherwise, its label is $0$ if the time series is in class $\theta_0$. We denote the two classes by $\{\theta_0,\theta_1\}$, which will be used as parameters for the time series models.  

A TSC algorithm learns a function with a parameter $\beta$ from data, 
\begin{equation}\label{eq:tsc_function}
	f_\beta(\bx) = z,  \quad \bx \in \R^{d(L+1)}, \, z \in [0,1], 
\end{equation}
such that the value $f_\beta(\bx)$ approximates the probability of $\bx$ being in class $\theta_1$, i.e., $\prob(\theta =\theta_1 \mid \bX= \bx )$.
This function leads to a classifier for any threshold $k\in (0,1)$:
\begin{equation}\label{eq:classifier_function}
	F(\bx,k) = \mathbf{1}_{R_k}(\bx),  \text{ where } R_k = \{\bx: f_\beta(\bx) >k\}, 
\end{equation}
where $R_k$ is called the \emph{acceptance region} to classify the time series $\bx$ as in class $\theta_1$ (equivalently, the \emph{rejection region} for the class $\theta_0$).


The confusion matrix of the binary classifier \eqref{eq:classifier_function} with $\theta_0$ as positive is shown in Table \ref{tab:confuationMat}. For a given threshold $k$, we have a false negative (FN) prediction if $F(\bx,k)=1$ while $\bx$ is in class $\theta_0$, and we have a false positive (FP) prediction if $F(\bx,k)=0$ while $\bx$ is in class $\theta_1$. The definitions of true positive (TP) and true negative (TN) are similar. The false negative rates (FNR) and the true negative rates (TNR) rates are the probabilities 
\begin{equation}\label{eq:alpha_k}
	\begin{aligned}
		\text{FNR}(k) & =\alpha_k^0 = \E[ F(\bx,k)\mid \theta_0)] = \prob(R_k \mid \theta_0)\approx  \frac{FN}{TP+FN}, \\
		\quad \text{TNR}(k) &= \alpha_k^1 =  \E[ F(\bx,k)\mid \theta_1)] =  \prob(R_k \mid \theta_1) \approx  \frac{TN}{TN+FP}, 
	\end{aligned}
\end{equation}
where the empirical approximations are based on the number of counts.

\begin{table}
	\begin{center}
		\caption{Confusion matrix of the classifier with $\{\theta_0\}$ being positive.  }\label{tab:confuationMat}
		\begin{tabular}{c |  cc | c c |}
			\hline 
			&   \multicolumn{2}{|c|}{Decision} &  \multicolumn{2}{|c|}{Rates/ Probability of errors }   \\ \hline
			Truth \textbackslash Decision       &   Accept $\theta_0$                  & Reject $\theta_0$   &  &     \\ \hline
			$\theta_0$ (Positive)        &  TP  & FN  &  TPR = $1-\alpha_k^0$          & FNR =$\alpha_k^0= \E[ F(\bx,k)\mid \theta_0)]$ \\ 
			$\theta_1$ (Negative)           &  FP   & TN  &  FPR = $1-\alpha_k^1$ &  TNR =$\alpha_k^1=  \E[ F(\bx,k)\mid \theta_1)]$\\				\hline				 
		\end{tabular}
	\end{center}
	* FN is also called type I error and FP is called type II error. The true positive rate (TPR) is $1-\alpha_k^0$ and the false positive rate (FPR) is $1-\alpha_k^1$.  
\end{table}

Two popular metrics evaluating the performance of the classifier are the \emph{Receiver operating characteristic} (ROC) curve and  \emph{accuracy}. The ROC curve is $(1-\alpha_k^0,1-\alpha_k^1)_{k\in (0,1)}$,  the curve of True Positive Rate (TPR, y-axis) versus False Positive Rate (FPR, x-axis), both parametrized by the threshold $k$ (see e.g., \cite{fawcett2006introduction} for an introduction). 
The ROC curve allows the user to define the threshold and measure the quality of a classifier by the \emph{area under the curve} (AUC). 
A rule of thumb is that the larger is the AUC, the better is the classifier.  The accuracy is defined by: 
\[ 
\text{ Accuracy(k) }= \frac{1- \alpha_k^0 + \alpha_k^1}{2}\approx  \frac{TP+TN}{TP+TN+FP+FN}, 
\] 
where the approximate equality becomes an equality when the sample sizes in the two classes are the same. The maximal accuracy is independent of the threshold:   
\begin{equation}\label{eq:k_opt}
  ACC_*  = \max_{0 \leq k\leq 1} \text{ Accuracy(k) } ,
\end{equation}

We will use AUC and the maximal accuracy to access the classifiers, because they are independent of a specific threshold. There are many other metrics to fit the goal of a specific field, i.e., choosing a threshold $k$ to increase the \emph{true positive rate} (TPR) (aka sensitivity, power, or recall) $1-\alpha_k^0$ or to control the false positive rate (FPR) $1-\alpha_k^1$ (aka specificity), or a balance 
balancing these needs \cite{james2013introduction}. 



\paragraph{Sampling errors in training and testing.} Sampling errors are present in the training and the testing data, thus they affect the accuracy of the classifier. The accuracy of the function learned in a classifier can be analyzed through mathematical and statistical learning theory (see e.g., \cite{CuckerSamle02,james2013introduction,devroye2013probabilistic}), and non-asymptotic error bounds are available to quantify the dependence on the data size based on concentration inequalities \cite{CZ07book,gine2015mathematical}. The sampling error in the testing stage, on the other hand, can be easily analyzed: the empirical approximation of the rates in \eqref{eq:alpha_k} have a sampling error of order $O(\frac{1}{\sqrt{m}})$ with $m$ being the number of test samples, as the next lemma shows (its proof is in Appendix \ref{sec:proof_sampling_error}). 

\begin{lemma}[Sampling error in FNR/TNR]\label{lemma:sampling_error_inTest}
For each classifier in \eqref{eq:classifier_function}, the sampling errors in the empirical approximations of the FNR and TNR rates in \eqref{eq:alpha_k} are of order $\frac{1}{\sqrt{m}}\sigma_{k,i} $ with $\sigma_{k,i}=\sqrt{\alpha_k^i(1-\alpha_k^i)}$ for $i=0,1$, where $m$ is the number samples in the test stage. Specifically, let $\{\bx_j\}_{j=1}^m$ be the test samples, and let $\widehat \alpha_{k,m}^i = \frac{1}{m} \sum_{j=1}^m F(\bx_j,k)$ conditional on $\theta_i$. Then, $\widehat \alpha_{k,m}^i$ converges in distribution to $\mathcal{N}(0,\sigma_{k,i}^2)$ as $m\to \infty$, and $\prob( \lvert \widehat{\alpha}_{k,m}^{i}  - \alpha_{k}^{i} \rvert >\epsilon ) \leq 2 e^{-\frac{m\epsilon^2}{2}}$ for any $\epsilon>0$ and $m>0$. 
\end{lemma}

However, the learning theory does not provide empirical optimality criteria for the performance of the classifier. The likelihood ratio test in the next section fills the gap.

\subsection{Hypothesis testing and the likelihood ratio test }\label{sec:LRT}

The hypothesis testing methods construct the classifier function in a statistical inference approach (see \cite[Chapter 8]{CasellaBerger01} and Section \ref{sec:hypothesis_testing} for a brief review). In particular, the \emph{Neyman-Pearson lemma} 
provides a powerful tool for analyzing the optimality of a binary classifier: it shows that the likelihood ratio test is a \emph{uniformly most powerful test} in the class of tests with the same level (see \cite[Chapter 8]{CasellaBerger01} and Section \ref{sec:hypothesis_testing} for a brief review). 

In hypothesis testing, we set the null hypothesis to be $H_0: \theta= \theta_0$ and the alternative hypothesis to be $H_1: \theta= \theta_1$, and we select a \emph{rejection region} $R_k$ with a threshold $k$ to reject $\theta_0$. Then, the classifier rejects the null hypothesis $H_0$ if the time series is in the rejection region $R_k$.  In other words, we get a false native (FN) if we mistakenly reject $H_0$ while the truth is $\theta_0$, and we get a true negative (TN) if we correctly reject $H_0$ when the truth is $\theta_1$. The false negative rate (FNR) and true negative rate (TNR) are the probabilities in \eqref{eq:alpha_k}. The major task in a hypothesis test is to select the rejection region $R_k$, particularly, to select $R_k$ with a tunable threshold $k$.

The \emph{likelihood ratio test} (LRT) is a general hypothesis testing method that is as widely applicable as maximum likelihood estimation.  
It determines the rejection region by statistics derived from the likelihood ratio. The commonly-used statistic is the log-likelihood ratio 
$$ l(\bx \mid \theta_1,\theta_0 )= \log \frac{p(\bx \mid \theta_1) }{p(\bx \mid \theta_0)}
$$
 of the time series data $\bx$. From this statistic, we can define a function approximating  the probability of $\bx$ being in class $\theta_1$, 
  which is a counterpart of $f_\beta(\bx)$ in \eqref{eq:tsc_function}:
$
f(\bx) = \frac{1}{e^{l(\bx \mid \theta_1,\theta_0 )}+1}.
$ 
Then, the classifier function for LRT is  $F(\bx,k) = \mathbf{1}_{R_k}(\bx)$ with the rejection region defined by  	
	\begin{equation}\label{eq:classifier_fn_LRT}
R_k^{{\scriptscriptstyle{LRT}}} 
=  \{\bx:  l(\bx \mid \theta_1,\theta_0 )  >c_k\},  \, c_k= \log \frac{k}{1-k}, 
\end{equation}
 for each threshold $k\in (0,1)$. 

The Neyman-Pearson lemma shows that the LRT is optimal in the sense that it has the lowest false positive rate among classifiers with a controlled level of false negative rate: 
\begin{theorem}[Neyman-Pearson Lemma]\label{thm:NP-lemma}
The LRT is a uniformly most powerful classifier. Specifically, let $\bx$ be a sample from one of two distributions with a likelihood ratio $ l(\bx \mid \theta_1,\theta_0 ) $ and assume that $\prob(\{\bx:  l(\bx \mid \theta_1,\theta_0 )  =k\}) = 0$. 
Then, the test with rejection region $R_k^{{\scriptscriptstyle{LRT}}}$ defined in \eqref{eq:classifier_fn_LRT} is uniformly most powerful.  That is, it has a false positive rate no larger than any other test with a measurable rejection region $R$ such that $\prob(R \mid   \theta_0)\leq  \prob(R_k^{{\scriptscriptstyle{LRT}}} \mid  \theta_0)$, i.e., 
\[
1- \prob( R \mid \theta_1)  \geq 1- \prob( R_k^{{\scriptscriptstyle{LRT}}} \mid \theta_1), \quad \forall R  \,\, s.t.\,\, \prob(R \mid  \theta_0)\leq  \prob(R_k^{{\scriptscriptstyle{LRT}}}\mid  \theta_0).
\]
\end{theorem}

As a result,  the LRT provides an ideal tool for analyzing the optimality of TSC algorithms. The ROC curve of the LRT classifier provides an upper bound for the ROC curve of any TSC classifier. Similarly, the LRT classifier's accuracy provides an upper bound for other classifiers. 

The LRT classifier can be readily applied to time series with a computable likelihood ratio, and there is no training stage. The transition densities determine the likelihood ratio when the time series is sampled from a Markov process.  Suppose that for each $\theta_i$,  the transition probability of the Markov process has a density function $p(x_{t_{l+1}} \mid x_{t_{l}}, \theta_i)$ for each $l$. Then, the probability density function of a data path $x_{t_0:t_L} $ conditional on $\theta_i$ is  
\[ p(x_{t_0:t_L} \mid \theta_i) = \prod_{l=0}^{L-1} p(x_{t_{l+1}} \mid x_{t_l}, \theta_i),\] 
and the log-likelihood ratio of the path is 
\begin{align}\label{eq:likelihood_discrete}
l( x_{t_0:t_L} \mid \theta_1,\theta_0) & = \log  \frac{p(x_{t_0:t_L} \mid \theta_1) }{p(x_{t_0:t_L}\mid \theta_0)} =   \sum_{l=0}^{L-1} \log \frac{p(x_{t_{l+1}} \mid x_{t_{l}}, \theta_1) }{p(x_{t_{l+1}} \mid x_{t_{l}}, \theta_0) } . 
\end{align}

However, the transition probabilities and the likelihood ratio are unavailable for most time series, except for a few simple examples such as Gaussian processes and linear models (see Section \ref{sec:analytical-LRT}). In particular, to benchmark the performance of TSC algorithms, it is desirable to have nonlinear time series datasets with varying lengths, temporal sampling frequencies, and dimensions. The diffusions defined by stochastic differential equations provide a large class of such Markov processes.

\subsection{Distinguishing diffusions}
Diffusion processes provide a large class of time series whose likelihood ratio can be accurately computed.
An It\^o diffusion is defined by a stochastic differential equation 
\begin{equation}\label{eq:diffusion_para}
dX_t = b_\theta(X_t) dt  + \sigma(X_t)dB_t, 
\end{equation} 
where $B_t$ is a standard $\R^d$-valued Brownian motion. Here for simplicity, we assume that both the diffusion coefficient $\sigma:\R^d\to \R^{d\times d}$ and the drift $b_\theta:\R^d\to \R^d$ with parameter $\theta$ are Lipschitz, and the diffusion satisfies the uniform elliptic condition $\sum_{1\leq i,j\leq d} c_i c_j \sigma_{ki} \sigma_{kj} (x) \geq \gamma \sum_{i}c_i^2$ with $\gamma>0$ for all $x$ and $c_i\in \R$. Beyond such diffusions, we can also consider It\^o processes with $b$ and $\sigma$ being general stochastic processes satisfying suitable integrability conditions \cite{oksendal2013_sde}. 

The likelihood ratio of a sample path $x_{t_0:t_L}$ of the diffusion $X_{t_0:t_L} $ satisfying \eqref{eq:diffusion_para} can be computed by numerical approximation of the transition probabilities. In particular, when the temporal sampling frequency is high, i.e., $\max_{l} \{\Delta t_l= t_{l+1}-t_l\}$ is small,  the  Euler-Maruyama scheme
\[ \Delta X_{t_l} = X_{t_{l+1}}- X_{t_l} \approx b_{\theta_i}(X_{t_l}) \Delta t_l + \sigma(X_{t_l}) \Delta W_l\  \]
yields an accurate approximation of the transition probability
\[
\widehat{p }(X_{l+1} \mid X_{l}, \theta_i) \propto e^{-\frac{1}{2 \Delta t }  \| \Delta X_{t_l} - b_{\theta_i}(X_{t_l}) \Delta t_l ) \|_\Sigma^2 }, 
\]
where 
$\Sigma(x) = \sigma\sigma^\top(x) \in \R^{d\times d}$ and $\|z\|_\Sigma^2 = z^\top \Sigma^{-1} z$ for any $z\in \R^d$. 
Using it in \eqref{eq:likelihood_discrete}, we obtain an approximate likelihood ratio: 
\begin{equation} \label{eq:discrete_lrt}
\begin{aligned}
	& \widehat{l}( X_{t_0:t_L} \mid \theta_1,\theta_0) \\
	=  & \sum_{l=0}^{L-1}\left( [b_{\theta_1} - b_{\theta_0}] (X_{t_l})^\top \Sigma(Y_s)^{-1}  \Delta X_{t_l} - \frac{1}{2} [\|b_{\theta_1}\|_\Sigma^2- \| b_{\theta_0}\|_\Sigma^2](X_{t_l}) \Delta t_l \right). 
\end{aligned}
\end{equation}
As the temporal sampling frequency increases, i.e., $\max_l \{t_{l+1}-t_l\} \to 0$, the above likelihood ratio converges to the likelihood ratio of the continuous path $X_{[0,T]}$. The limit ratio is the Radon-Nikodym derivative between the two distributions of the path, as characterized by the Girsanov theorem 
(see Section \ref{sec:diffusion_Girsanov}):
\begin{equation}\label{eq:likelihood}
\begin{aligned}
& l( X_{[0,T]} \mid \theta_1,\theta_0 ) \\ 
 = & \int_0^T [b_{\theta_1} - b_{\theta_0}] (Y_s)^\top \Sigma(Y_s)^{-1} dY_s - \frac{1}{2}\int_0^T [\| b_{\theta_1} \|_\Sigma^2- \| b_{\theta_0}\|_\Sigma^2](X_t) dt. 
\end{aligned}
\end{equation}

There are three advantages to benchmarking TSC algorithms by diffusions. First, the LRT of the diffusion processes provides the theoretical optimal rates, which can be used to detect overfitting when training TSC classifiers. Second, the diffusions provide a large variety of testing time series data, whose length, sampling frequency, dimension, and nonlinearity can vary as needed. Third, numerical approximation can efficiently compute the likelihood ratio between diffusion processes as in \eqref{eq:discrete_lrt}.

\section{Examples with analytical likelihood ratios}\label{sec:analytical-LRT}
The likelihood ratio can be computed analytically for Brownian motions with constant drifts and the Ornstein-Uhlenbeck (OU) processes. In particular, these two examples offer insights into how classification accuracy depends on the temporal sampling frequency, length of paths, randomness, and the dimension of the time series data.  
\subsection{Brownian motions with constant drifts}	
Let  $(X_t, t\geq 0)$ be an $\R^d$-valued Brownian motion with a constant drift:  
\begin{equation}\label{eq:constant_drift}
	\begin{aligned}
		dX_t = \theta dt+ \sigma dB_t,  \quad  \Leftrightarrow \quad  X_t  & = X_{0} + \theta t + \sigma B_t, \\
	\end{aligned}
\end{equation}
where $\theta\in\{\theta_0, \theta_1\} \subset \R^d$ and the process $(B_t, t\geq 0)$ is the standard Brownian motion starting at $0$. 
Then, the exact log-likelihood ratio in \eqref{eq:likelihood_discrete} for a given sample path $X_{t_0:t_L}$ is  
\[
l(X_{t_0:t_L} \mid \theta_1,\theta_0) =  \sigma^{-2}\left[ (\theta_1-\theta_0)^\top(X_{t_L} - X_{t_0})- \frac{1}{2}(\lvert \theta_1 \rvert^2-\lvert\theta_0\rvert^2)(t_L-t_0)\right]. 
\] 
Note that $X_{t_L} - X_{t_0} = \theta(t_L-t_0) + \sigma (B_{t_L} - B_{t_0})$ for each $\theta$. 
Thus, conditional on the hypotheses $\theta=\theta_0$ and $\theta=\theta_1$, the likelihood ratios have distributions 
\begin{equation*}
	\begin{aligned}
		\text{Hypothesis } \theta=\theta_0:\quad    & l(X_{t_0:t_L} \mid \theta_1,\theta_0) \sim - m_l + v_l Z, \\
		\text{Hypothesis } \theta=\theta_1:\quad    & l(X_{t_0:t_L}  \mid \theta_1,\theta_0) \sim m_l + v_l Z ,
	\end{aligned}
\end{equation*}
where $Z$ is a standard Gaussian random variable and 
\[ m_l = \frac{1}{2} \lvert\theta_1-\theta_0\rvert^2(t_L-t_0), \quad v_l =\sigma \lvert\theta_1-\theta_0\rvert \sqrt{t_L-t_0}.\] 		
Let the rejection region be $R_k= \{X_{t_0:t_L}: l(X_{t_0:t_L} \mid \theta_1,\theta_0) >c_k \} $ with $c_k= \log\frac{k}{1-k}$ as defined in \eqref{eq:classifier_fn_LRT}.
Then, the false negative rate (FNR) and the true negative rate (TNR) of the LRT  are
\[
\begin{aligned}
	\text{FNR}(k)=	\alpha_k^0 & = \prob( x_{t_0:t_L} \in R_k \mid \theta_0 ) =  \prob(  Z > c_k v_l^{-1} + m_l v_l^{-1}  ) \\
	\text{TNR}(k)=     \alpha_k^1 & = \prob( x_{t_0:t_L} \in R_k \mid \theta_1) =  \prob(  Z> c_k v_l^{-1} - m_l v_l^{-1} ).
\end{aligned}
\]
Then,  the accuracy $\frac{1}{2}(1-\alpha_k^0 +\alpha_k^1)$ is
$ACC_k = \frac{1}{2} + \frac{1}{2} \prob \left( -m_l v_l^{-1} < Z- c_k v_l^{-1}  <  m_l v_l^{-1}\right) $.
Since $Z$ is centered Gaussian, the threshold maximizing the accuracy is $k_*= \argmax{k\in (0,1)} ( ACC_k ) = 0$. As a result, the maximal accuracy is 
\begin{align*}
	\quad ACC_{*}  & = \frac{1}{2} + \frac{1}{2} \prob \left( -m_l v_l^{-1} < Z <  m_l v_l^{-1} \right) \\
	& = \frac{1}{2} + \frac{1}{2} \prob \left( - \frac{1}{2\sigma} \lvert\theta_1-\theta_0\rvert \sqrt{  (t_L - t_0) } < Z <  \frac{1}{2\sigma}  \lvert\theta_1-\theta_0\rvert \sqrt{  (t_L - t_0) } \right). 
\end{align*}

The above FNR and TNR rates and the maximal accuracy depend on three factors: the path length $t_L-t_0$, the scale of the noise $\sigma$ (which affects the variance of the time series), and the distance $ \lvert\theta_1-\theta_0\rvert$ (which depends on the dimension $d$). As either $\sqrt{t_L-t_0}$, $ \lvert\theta_1-\theta_0\rvert$, or $\sigma^{-1}$ increases, the maximal accuracy increases. For example, when $ \theta_0 = a_0 [1,...,1]^\top $, and $ \theta_1 = a_1 [1,...,1]^\top$, $ \lvert\theta_1-\theta_0\rvert= d^{1/2}$, and the maximal accuracy is 
$$ ACC_{k_*}  =1- \prob(\lvert Z \rvert \geq \frac{1}{2\sigma}\lvert a_1-a_0 \rvert  \sqrt{d  (t_L - t_0) } ) .$$  

These rates and the maximal accuracy do not depend on the temporal sampling frequency of the time series because the likelihood ratio is exact. However, the temporal sampling frequency will affect the accuracy when the likelihood ratio is approximated numerically as in \eqref{eq:discrete_lrt}, particularly for nonlinear time series; see the numerical examples in Section \ref{sec:num}.

\subsection{Ornstein-Uhlenbeck processes}	
	Consider two $\R^d$-valued OU processes with parameters $\theta\in \{\theta_0, \theta_1\} \subset \R$:
	\begin{equation}\label{eq:ou}
		dX_t  = \theta X_t dt + \sigma dB_t   \,  \Leftrightarrow \,  X_{t+\Delta t}  = e^{ \theta \Delta t} X_t + \sigma \int_t^{t+\Delta t} e^{\theta (t+\Delta t-r)} dB_r 
	\end{equation}
for each $t>0$,	where $(B_t, t\geq 0)$ is an $\R^d$-valued standard Brownian motion and $\sigma>0$ is a constant. Then, conditional on $X_t$ and $\theta_i$, the random variable $X_{t+\Delta t}$ has a distribution $\mathcal{N} \left( X_t e^{\theta_i \Delta t}, \frac{\sigma^2}{2 \theta_i} \left( 1 - e^{2\theta_i \Delta t} \right) I _d \right)$, and the transition probability density of this Markov process is 
	\[
	p(x_{t+\Delta t} \mid x_{t}, \theta_i) = (2\pi \sigma_{i,\Delta t}^2)^{-d/2}  \exp{\left( - \frac{1}{\sigma_{i,\Delta t}^2} \|x_{t+\Delta t} - e^{2\theta_i \Delta t} y_t\|^2  \right)} 
	\]
	with $\sigma_{i,\Delta t}^2 = \frac{\sigma^2}{2 \theta_i} \left( 1 - e^{2\theta_i \Delta t } \right)$. 
	Let  $X_{t_0:t_L}$ be a discrete path with $t_l= l\Delta t$ for $0\leq l\leq L$. By the Markov property, the logarithm probability density of $X_{t_0:t_L}$  conditional on $\theta_i$ is  
	\begin{equation*}
		\begin{aligned}
	  \log p (X_{t_0 : t_L} \mid \theta_i) = C - \frac{d L}{2} \log(\sigma_{i, \Delta t}^2) - \frac{1}{ 2 \sigma_{i, \Delta t}^2 }  \sum_{l=0}^{L-1} \| X_{t_{l+1}} - e^{\theta_i \Delta t} X_{t_{l}} \|^2, 
		\end{aligned}
	\end{equation*}
	where $C$ is a constant. Thus, the log-likelihood ratio in \eqref{eq:likelihood_discrete} is 
	\[
	\begin{aligned}
		 l(X_{t_0:t_L} \mid \theta_1,\theta_0)  
		=  & \frac {d L} {2} \log \left( \frac{\sigma_{0, \Delta t}^2}{\sigma_{1, \Delta t}^2} \right) \\
		+ & \frac{1}{ 2 }  \sum_{l=0}^{L-1} \left( \frac{\| X_{t_{l+1}} - e^{\theta_0 \Delta t} X_{t_{l}} \|^2}{\sigma_{0, \Delta t}^2 } - \frac{\| X_{t_{l+1}} - e^{\theta_1 \Delta t} X_{t_{l}} \|^2}{\sigma_{1, \Delta t}^2 } \right).	
	\end{aligned}
	\]
	
	Let the rejection region be $R_k=\{X_{t_0:t_L}: l(X_{t_0:t_L} \mid \theta_1,\theta_0) > c_k  \}$. Note that conditional on $\theta_0$, $N_l := \frac{1}{\sigma_{0, \Delta t} } \left( X_{t_{l+1}} - e^{\theta_0 \Delta t} X_{t_{l}} \right)$ has a distribution $\mathcal{N} (0, I_d) $ for each $l$, and $X_{t_{l+1}} = e^{\theta_0 \Delta t} X_{t_{l}} +\sigma_{0, \Delta t}  N_l $. Then, with $Y_l= ( e^{\theta_1 \Delta t} - e^{\theta_0 \Delta t}) X_{t_{l}} + \sigma_{0, \Delta t} N_l $,  the false positive rate (FNR) is  
	\[
	\begin{aligned}
		\alpha_k^0 = & \prob \left( l(X_{t_0:t_L} \mid \theta_1,\theta_0) > c_k \mid \theta_0 \right)
		\\
		= & \prob \left(   \sum_{l=0}^{L-1} \big[ \| N_l \| ^2 - \sigma_{1, \Delta t}^{-2}   \|  Y_l\|^2 \big] > 2 c_k - d L \log \left( \frac{\sigma_{0, \Delta t}^2}{\sigma_{1, \Delta t}^2} \right)  \right), 
	\end{aligned}
	\]
	with $N_l\sim \mathcal{N} (0, I_d) $.  Similarly, denoting $Y_l' =( e^{\theta_0 \Delta t} - e^{\theta_1 \Delta t}) X_{t_{l}} + \sigma_{1, \Delta t} N_l  $,  we can compute the true negative rate (TNR)
	\[
	\begin{aligned}
		& \alpha_k^1 
		= \prob \left(    \sum_{l=0}^{L-1} \left[\sigma_{0, \Delta t}^{-2}  \|  Y_l'\|^2 -   \| N_l \| ^2 \right] > 2c_k - d L \log \left( \frac{\sigma_{0, \Delta t}^2}{\sigma_{1, \Delta t}^2} \right)  \right). 
	\end{aligned}
	\]
	The optimal threshold $k = \argmax{k}\frac{1}{2}(1-\alpha_k^0 +\alpha_k^1)$ depends on the various factors of the time series, so is the maximal accuracy. The numerical examples in Section \ref{sec:num} shows that the maximal accuracy increases as either $d$ or $L$ increases.

\section{Benchmark design: example diffusions 
} \label{sec:examples}
We demonstrate the construction of diffusions for TSC benchmarking with three representative examples. 
In each example, the procedure is straightforward: first, we construct pairs of diffusions through varying the drifts. Then, we generate data from these diffusions, and compute the statistics of LRT, which will be used as a reference for the performance of the state-of-the-art machine learning TSC algorithms in the next section.  
\subsection{Diffusions with different drifts}
Nonlinear diffusions can be constructed by varying the drifts $\{b_{\theta_i}\}_{i=0,1}$:  
\begin{equation}\label{eq:sde}
\begin{aligned}
	dX_t & =b_{\theta_i}(X_t)\, dt + \sigma(X_t) dB_t,  \quad b_{\theta_i}(X_t)=\sum_{j=1}^J \theta_{i,j}\phi_j(X_t), 
\end{aligned}
\end{equation}
where $X_t\in \R^d$, $\theta_i = (\theta_{i,1},\ldots, \theta_{i,J}) \in \R^J$ are the parameters, $\{\phi_j\}$ are \emph{pre-specified} basis functions, and  $(B_t, t\geq 0)$ is the standard Brownian motion in $\R^d$. Here the diffusion coefficient $\sigma(X_t)$ is the same for the two diffusions, representing either a multiplicative noise (when it depends on the state) or an additive noise (when it is a constant). To test the optimality of the TSC algorithms, we consider three pairs of nonlinear diffusions: gradient systems with different potentials, SDEs with linear and nonlinear drifts, and high-dimensional interacting particle systems with different interaction kernels.

\begin{myexample}[Different potentials] \label{example:potentials}
Consider two gradient systems with different potentials: a double-well potential $V_{\theta_0}(x) =\frac{1}{2} (|x|^2-1)^2$ and a single flat well-potential $V_{\theta_1}(x) =\frac{1}{4} |x|^4$: 
\begin{equation*}
	\begin{aligned}
		dX_t  & = - \nabla V_{\theta_i}(X_t) dt +  dB_t.\\
	\end{aligned}
\end{equation*}
Writing them in the parametric form $V_{\theta_i}(X) = \sum_{j=0}^4 \theta_{i,j} |x|^j$ with $\theta_{i}=(\theta_{i,1},\ldots,\theta_{i,4})$, we have $\theta_0 =\frac{1}{4}(1,0,-2,0,1) $ and $\theta_1 =( 0, 0, 0,0,\frac{1}{4}) $. 
\begin{figure}[h] 
	\centering
  \includegraphics[width =.9\textwidth]{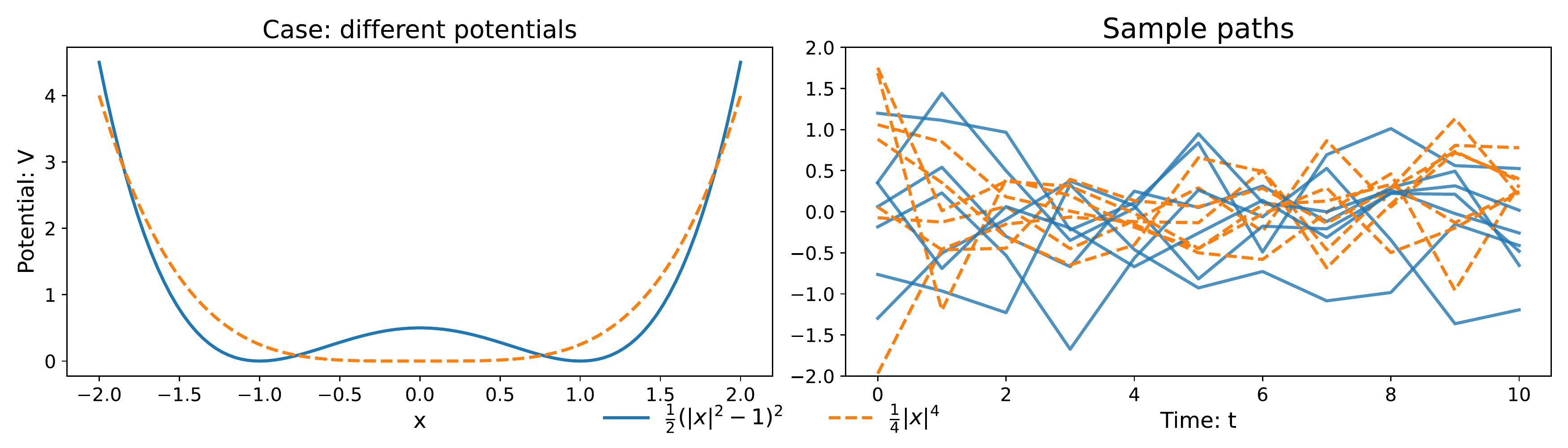} \ifarXiv \vspace{-3mm} \fi
	\caption{Different potentials in Example \ref{example:potentials} and a few sample paths.  }	\label{fig:2potentials}  
\end{figure}
\end{myexample}
The double-well potential is a widely-used prototype model for systems with metastable states \cite{pavliotis2014stochastic}.  These two potentials are visually different; see Figure \ref{fig:2potentials} (left). Each potential is confining and leads to an ergodic process with a stationary distribution. Thus, 
long sample paths that explore the full landscape of the potentials can distinguish the diffusions from the empirical densities. However, the short sample paths look similar and are difficult to distinguish, as shown in Figure \ref{fig:2potentials} (right).  

\begin{figure}[h]  \vspace{-2mm}
\centering
  \includegraphics[width =.9\textwidth]{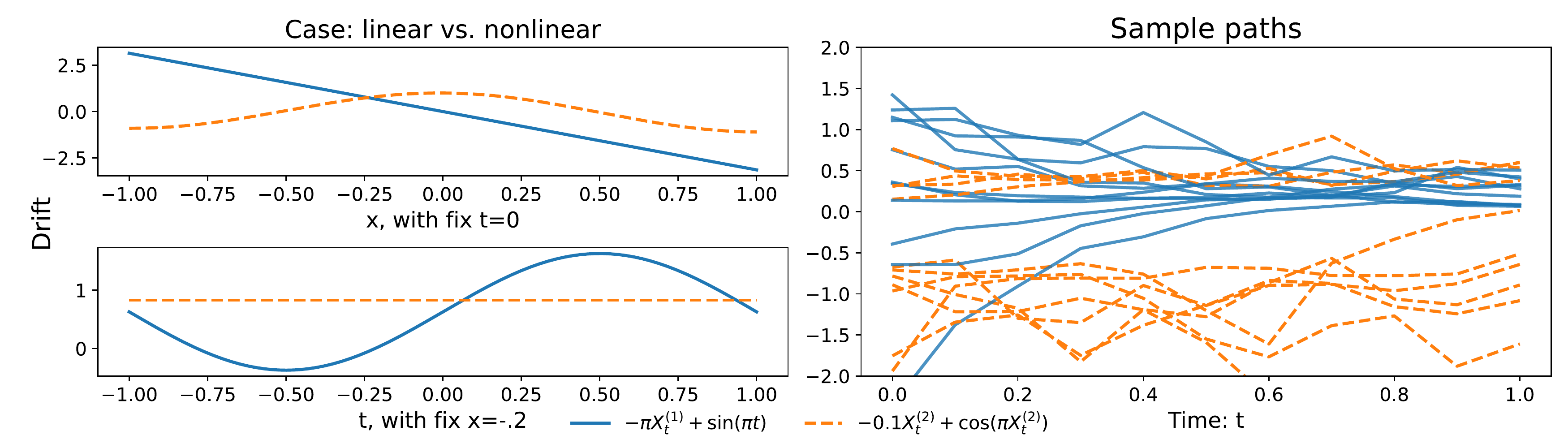}  \ifarXiv \vspace{-3mm} \fi
\caption{Linear v.s.~nonlinear drifts in Example \ref{example:linear_nonlinear}  and a few typical sample paths.  }
\label{fig:example_4_2} 
\end{figure}
\begin{myexample}[Linear v.s.~nonlinear drifts]\label{example:linear_nonlinear} 
Consider two 1D It\^o processes 
\begin{equation*}
	dX_t = b_\theta(t,X_t) dt + X_t dB_t
\end{equation*}
with linear and nonlinear drifts $b_{\theta_0}(t, x) = -\pi x + \sin(\pi t)$ and $b_{\theta_1}(t, x) = -0.1 x + \cos(\pi x)$, which can be written as $ b_{\theta_i}(t, x) = \theta_{i,1} x +  \theta_{i,2}  \cos(\pi x)) + \theta_{i,3} \sin(\pi t)) $ with  $\theta_0 =(-\pi,0,1) $ and $\theta_1 =( -0.1, 1,0) $. 
\end{myexample}
The two drift functions are visibly different, since $b_{\theta_0}(t, x)$ is linear in $x$ and the other is nonlinear in $x$. Their sample paths are also visually different: the sample paths of $b_{\theta_0}$ are smoother than those of $b_{\theta_1}$'s (they decay faster); see Figure \ref{fig:example_4_2}. Thus, we expect that all TSC algorithms can achieve a high accuracy.

\begin{myexample}[Interacting particles]\label{example:IPS}
Consider a system with $N$ interacting agents with $X_t^j\in \R^{d_1}$ denoting the position or opinion of the $j$-th agent at time $t$. Suppose that the agents interact with each other according to the following stochastic differential equation: 
\begin{equation*}
	d X_{t}^j= \frac 1 N \mathop \sum_{j=1}^N \phi_\theta(\|X_{t}^j-X_{t}^i\|)(X_{t}^j-X_{t}^i) + \sigma dB_t^j, 
\end{equation*}
where $\phi_\theta:\R^+\to \R$ is the interaction kernel, $\{B_t^j, j=1,\ldots, N\} $ are independent standard Brownian motions, and $\sigma>0$ is a scalar for the strength of the stochastic force. We will consider two types of interaction kernels (see Figure {\rm \ref{fig:example_4_3} (left)})
\begin{equation*}
	\phi_{\theta_0}(r) =  \left\lbrace\begin{aligned}
		& 0.2 , \ \ \ \ \ \ \ &r \in [0,\sqrt{2}),
		\\
		& 2 , \ \ \ &r \in [\sqrt{2},2),
		\\
		& 0 , \ \ \ & r \in [2,\infty).
	\end{aligned}\right.  \ \ \ \ \ \ \  
	\phi_{\theta_1}(r) =  \left\lbrace\begin{aligned}
		& 2 , \ \ \ \ \ \ \ &r \in [0,\sqrt{2}),
		\\
		& 0.2 , \ \ \ &r \in [\sqrt{2},2),
		\\
		& 0 , \ \ \ & r \in [2,\infty).
	\end{aligned}\right. 
\end{equation*}
This system leads to high-dimensional data, with $X_t = (X_t^1,\ldots, X_t^N)\in \R^{d}$ with $d=d_1N$. We will consider $d_1=2$ and $\sigma=1$ with $N$ varying to change the dimension of the system.  
\end{myexample}
\begin{figure}[h] \vspace{-3mm}
\centering
  \includegraphics[width =.9\textwidth]{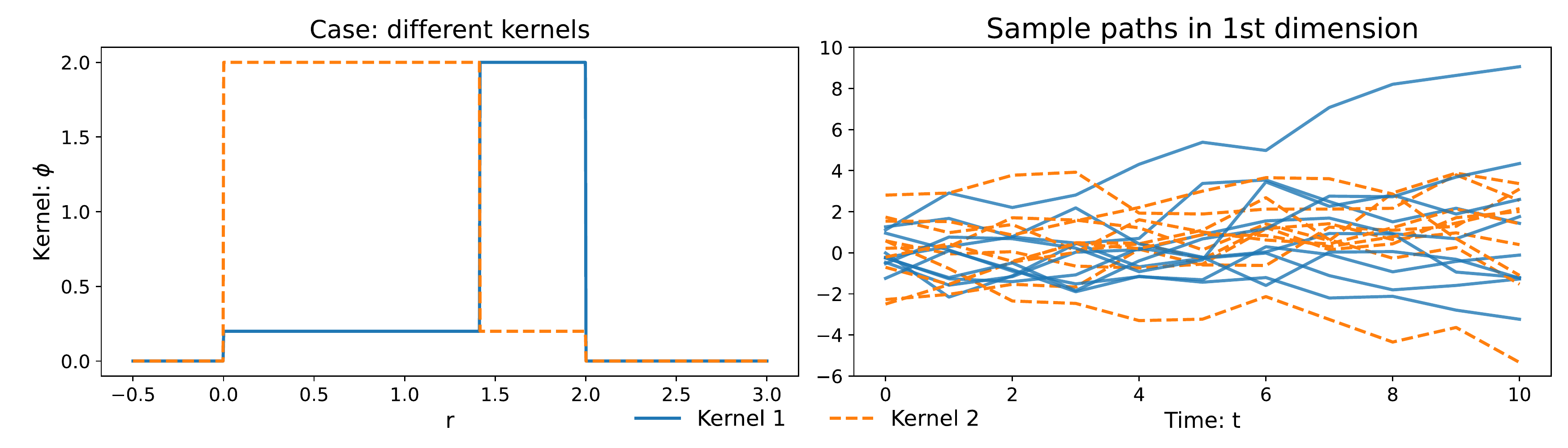}  \ifarXiv \vspace{-3mm} \fi
\caption{Interaction kernels in Example \ref{example:IPS} and sample paths of the 1st dimension of an agent.}
\label{fig:example_4_3}  
\end{figure}

Such interacting particle systems have been increasingly studied because of their wide-range of applications in biology, engineering, and social science (see e.g.,\cite{Krause2000,bell2005_ParticlebasedSimulation,MT14,LMT21}). The difference between the two kernels is the strength of interaction between ``far'' and ``close'' neighbors:  the kernel $\phi_{\theta_1}$ makes the close neighbors interact stronger than those far away, whereas the kernel $\phi_{\theta_0}$ makes the far neighbors interact stronger than those nearby. Then, the dynamics of the two systems are different, and it is shown in \cite{MT14} that the more heterophilious kernel $\phi_{\theta_0}$ enhances consensus when there is no stochastic force (i.e., the systems is deterministic). As a result, it is relatively easy to distinguish the two diffusions when the stochastic force is small. On the other hand, when the stochastic force is relatively large (e.g., $\sigma=1$), the sample paths of the agents in the two systems are similar (Figure \ref{fig:example_4_3} (right)), making the classification a difficult task.

\subsection{Data generation and the LRT benchmarks}

The simulated diffusion processes allow us to test the dependence of classification performance on three parameters: path length in time $t_L$, the dimension $d$ of the state, and temporal sampling frequency (by varying the time gap $\Delta t$). We test each of the three parameters with four values using two diffusion models, thus in total we generate 24 datasets in 6 cases with these parameters specified in Table  \ref{tab:settings}.   
\begin{table}
\caption{Settings of the time series data in numerical tests.}\label{tab:settings} \ifarXiv \vspace{-2mm}\fi
{\centering
\begin{tabular}{ l | l l l l l  }		
	Model & $ d $ & $ L $ & $ t_L $,  $ \Delta t $   \\ 
	\hline	
	{\bf a)} Constant drifts  & $1$ & $\{10, 20, 40, 80\}$ & $\{1, 2, 4, 8\}$,  $ 0.1 $ \\
	{\bf b)} Different potentials & $1$ & $\{20, 40, 80, 160\}$ & $\{2, 4, 8, 16\}$, $ 0.1 $ \\
	{\bf c)} OU processes & $\{1,2,4,8\}$ & $20$ & $ 2 $ , $ 0.1 $ \\
	{\bf d)} Interacting particles  & $\{6,12,24,48\}$ & $20$ & $ 2 $, $ 0.1 $ \\
	{\bf e)} Linear v.s.~nonlinear   & $1$ & $\{5, 10, 20, 40\}$ & $ 1$,  $0.1\times\{2, 1, 0.5, 0.25\}$  \\	
	{\bf f)} Interacting particles & $24$ & $\{ 10, 20, 40, 80 \}$ & $ 4$, $0.1\times \{ 4, 2, 1, 0.5 \} $ \\
\end{tabular} \\
}
 {\footnotesize $^*$ The models ``Constant drifts'' and ``OU processes''  are defined in Equations \eqref{eq:constant_drift} and  \eqref{eq:ou}, and the models ``Different potentials'', ``Interacting particles'' and ``Linear v.s.~nonlinear'' are defined in Examples \ref{example:potentials}--\ref{example:IPS}.
 } 
\end{table}

In each dataset, the training data consists of $M=2000$ sample trajectories $ \{X^{(m)}_{t_0:t_L}\}_{m=1}^M $ of the pair of $\R^d$-valued diffusions with $\theta\in \{\theta_0,\theta_1\}$, 1000 paths for each of the pair. Here the time instances are $t_l= l\Delta t$, and these data paths are downsampled from the solutions of the SDEs simulated by the Euler-Maruyama scheme with a fine time step $\delta =0.01$. For example, the path with $\Delta t =0.1$ makes an observation every $10$ time steps from the fine simulated solution. 
The initial conditions $\{X^{(m)}_{t_0}\}_{m=1}^M$ are sampled from the standard normal distribution in $\R^d$. Each sample path is augmented with its time grid $ t_0:t_L $ with $t_0=0$.

For each dataset, we obtain two types of LRT benchmarks by computing the LRT in two ways: one using the fine paths and the other using the time series dataset, both compute the likelihood ratio using the Euler-Maruyama approximation in \eqref{eq:discrete_lrt}. Since there is no need of training, each classifier makes a prediction directly on the whole dataset of $M$ paths, and returns a single ROC curve, AUC and ACC$_*$, which will be used as references. 

The LRT classifier using the fine solution is called ``\emph{LRT hidden truth''}, and it provides the optimal classification rates by the Neyman--Person lemma (see Theorem \ref{thm:NP-lemma}). The other LRT classifier using the training data is called ``\emph{LRT numerical}''. It does not use the hidden fine path, but it uses the diffusion model information that are not used by the TSC algorithms. It has a relatively large numerical error when the SDE is nonlinear, particularly when the observation time interval $\Delta t$ is much larger than the simulation time step $\delta$. Thus, it provides a lower baseline for the TSC algorithms.  The two LRT benchmarks are the same when the time series are samples of a Gaussian process from a linear SDE, e.g., the cases of Brownian motions with constant drifts and OU processes.

\subsection{Discussions on benchmark design}
The LRT benchmark design has two main components: selection of the diffusion processes and generation of simulated data. In addition to the examples in Table \ref{tab:settings}, there is a large variety of diffusion processes from stochastic differential equations in the form of \eqref{eq:sde}, such as gradient systems and stochastic Hamiltonian systems \cite{pavliotis2014stochastic,oksendal2013_sde}. The two diffusions should have the same diffusion coefficient, so that the likelihood ratio can be computed based on the Girsanov theorem. 

To generate simulated data, we recommend using the Euler-Maruyama scheme so that the likelihood ratio of the fine trajectory is exact. The time series data are downsampled from the fine trajectories. It is helpful to compute two LRT benchmarks, one using the fine trajectories and the other using the downsampled data, to provide an optimality benchmark and a lower baseline benchmark. In particular, the optimality benchmark can detect the overfitting of a TSC algorithm in the training stage. 

Four parameters can be tuned to adjust the theoretical classification accuracy: the time length of paths, the dimension, the temporal sampling frequency, and the strength of the driving noise (as suggested by the analysis in Section \ref{sec:analytical-LRT}). The time length of paths and the dimension affect the effective sample size and hence the classification rates. The temporal sampling frequency affects the LRT baseline but it may have a limited effect on the model-agnostic TSC algorithms. At last, a large noise dims the signal from the drifts, thus lowering the accuracy of classification.

\section{Benchmarking random forest, ROCKET and ResNet}\label{sec:num}
\subsection{Random forest, ROCKET, and ResNet}
We benchmark three scalable TSC methods: random forest \cite{breiman2001random}, ROCKET \cite{dempster2020rocket}, and ResNet \cite{wang2017time}. They have been shown to be state-of-the-art in recent review papers \cite{bagnall2017great,ruiz2021great,ismail2019deep}. In particular, the most recent review \cite{ruiz2021great} compares 11 multivariate time series classifiers that are top-performers in \cite{bagnall2017great,ismail2019deep},  including both non-deep learning methods (including  ROCKET and HIVE-COTE (Hierarchical Vote Collective Of Transformation-based Ensembles) \cite{lines2016hive}) and deep learning methods (including ResNet and InceptionTime \cite{ismail2020inceptiontime}), using 26 UEA archive datasets \cite{bagnall2018uea}. The recommended method is ROCKET due to its high overall accuracy and remarkably fast training time. 

\paragraph{Random Forest.} The random forest (RF) is an ensemble learning technique that combines a large number of decision trees, and it is applicable to both classification and regression. The original RF described by \cite{breiman2001random} is a classifier consisting of a collection of tree-structured classifiers $\{f(\bx,\beta_i)\}_{i=1}^{n_T}$ with independent identically distributed parameters $\beta_i$ and each tree casts a unit vote for the input $\bx$ to be in a class. These votes lead to a function $f_\beta(\bx) = \frac{1}{n_T}\sum_{1=1}^{n_T} f(\bx,\beta_i)$ approximating the probability of $\bx$ being in the class (i.e., the probability $\prob(\theta= \theta_1\mid \bx)$ of $\bx$ in the class  $\theta_1$ in our notation in Section \ref{sec:tsc_learning}). The classifier function with a threshold $k$ is $F(\bx,k)$ as in \eqref{eq:classifier_function}. 
It is user-friendly with only a few parameters easy to tune to achieve robust performance, and its performance is comparable to other classifiers such as discriminate analysis, support vector machine and neural networks \cite{liaw2002classification,probst2017tune}. 

We use the default HalvingRandomSearchCV strategy in scikit-learn \cite{scikit-learn} to search for parameter values in the ranges listed below. 
\begin{center}
\begin{tabular}{ l | l l l l l  }		
	 & \# of trees & max depth & max features & min SS & bootstrap   \\
	\hline	
	RF & $ \{ 10:100 \} $ & $ \{ 3, \text{None} \} $ & $ \{ 1 : 11 \}$ & $ \{ 2 : 11 \} $ & $ \{ \text{True}, \text{False} \} $ ,  
\end{tabular}
\end{center}
where ``min SS'' represents minimal samples split, and the quality of a split is measured by the Gini index 
Note that number of trees is medium so as to have a comparable computational cost with other methods.

\paragraph{ResNet.} The deep residual network (ResNet) for time series classification \cite{wang2017time} is a network with three consecutive blocks, each comprised of three convolutional layers, followed by a global average pooling layer and a final dense layer with softmax activation function. The major characteristic is that the three consecutive blocks are connected by residual ``shortcut'' connections, enabling the flow of the gradient directly through them, thus reducing the vanishing gradient effect \cite{he2016deep}. It outperforms other deep learning time series classifiers in \cite{ismail2019deep}, especially for univariate datasets \cite{wang2017time}. 

We maintain all hyper-parameter settings from \cite{ismail2019deep}. 
\begin{center}
\begin{tabular}{ l | l l l l l }		
	Structure & layers & activate & normalize & residue & dropout \\
	\hline	
	ResNet & 9+2 & ReLU & batch & between blocks & none. \\
\end{tabular}
\end{center}
There are nine convolution layers in the three blocks, each with the ReLU activation function that is preceded by a batch normalization operation. The number of filters in each convolution layer is 64 in the first block; while the number is 128 for the second and third blocks. In each residual block, the kernel size (or the length of the filter) is set to 8, 5 and 3 respectively for the first, second and third convolution. 
The optimization settings are also similar to \cite{ismail2019deep}: 
\begin{center}
\begin{tabular}{ l | l l l l l l }		
	Training & optimizer & rlr & epochs & batch & learning rate & weight decay \\
	\hline	
	ResNet & Adam & yes & $150$ & $16$ & $0.001$ & $0.0$ \, ,  \\
\end{tabular}
\end{center}
where `` rlr" means that the learning rate is reduced by half if the model's training loss has not improved for $5$ consecutive epochs with a minimum learning rate set to $ 0.0001 $. Here we set the epochs to 150 to have a computational cost comparable with other methods while maintaining accuracy. 

\paragraph{ROCKET.} The ROCKET (Random Convolutional Kernel Transform) \cite{dempster2020rocket} is the current state-of-the-art multivariate time series classifier \cite{ruiz2021great}. It uses random transformations followed by a linear classifier (ridge regression or logistic regression). 
In the transformation part, a large number of random convolution kernels are applied to each time series, each kernel producing a feature map. From each of these feature maps, two features are extracted: the maximal value and the proportion of positive value (ppv). Thus, each random kernel extracts two features from each time series. The linear classifier then makes classification based on these features. 

We keep the default setting for ROCKET in the \emph{sktime} reposIt\^ory \footnote{\url{https://github.com/alan-turing-institute/sktime/blob/master/sktime/transformers/series_as_features/rocket.py.}}, and we use the ridge regression (the parameter regularization strength $\alpha$ is searched by the build-in function RidgeCV).
The randomness comes from the kernel's parameters: length, weights, bias, dilation, and padding: 
\begin{center}
\begin{tabular}{ l | l l l l l }		
	Kernel & length & weight & dilation & padding or not & stride \\
	\hline	
	ROCKET & $ \{ 7, 9, 11 \} $ & $\mathcal{N} (0,1) $ & $\floor{2^x}$ & equal probability & $1$.
\end{tabular}
\end{center}
Here $ x \sim \mathcal{N} (0,A)$ with $ A=\log_2 \frac{l_{\text{input}}-1 }{l_{\text{kernel}}-1} $, where $l_{\text{input}}$ and $l_{\text{kernel}}$ are the lengths of the time series and the kernel.  
The number of kernels is set to $10000$, resulting in $20000$ features for each time series.

\subsection{ROC curves in a typical test}\label{sec:roc}
We compare the performance of these TSC algorithms with the LRT benchmarks in three statistics: the ROC curve in a typical test, the box-and-whisker plots of AUC (area under the ROC curve) and the optimal accuracy (denoted by ACC$_{*}$) in $40$ different runs. In each run, we train the algorithms using randomly sampled 3/4 of the data paths and use the rest 1/4 of the data for prediction test. Thus, each algorithm is trained using $M_{\text{training}} = \frac{3}{4}M = 1500$ sample paths and the rates in prediction are computed using $\frac{1}{4} M = 500$ sample paths. By Lemma \ref{lemma:sampling_error_inTest},  each prediction rate has a standard deviation at the scale of $\frac{0.5}{\sqrt{500}}= 0.02 $. Thus, two algorithms perform similarly if the difference between their rates are within the sampling error of 0.04 (in two standard deviations).  

\begin{figure} [htb] 
\centering 
\includegraphics[width =1\textwidth]{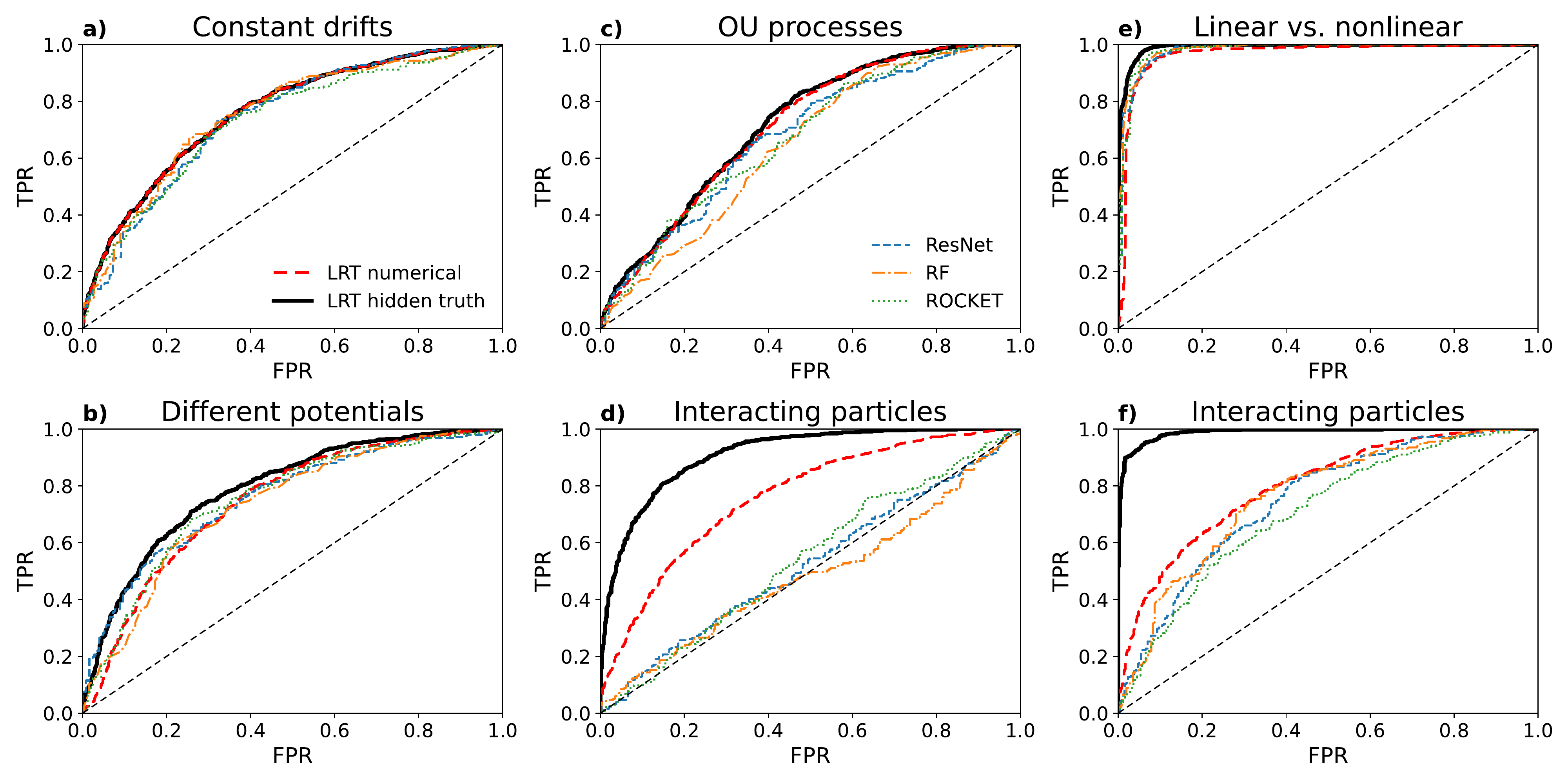}  \vspace{-6mm}
\caption{ROC curves in a typical test in each of the 6 cases (each using the first of the settings in Table \ref{tab:settings}). The three algorithms achieve the optimal LRT in Case (a), and they are in-between the two LRT benchmarks in Cases (b,e). 
They are suboptimal in comparison with the ``LRT numerical'' in Cases (c,f) and they have unsuccessful classification in Case (d).     } 
\label{fig:ROC}
\end{figure}

Figure \ref{fig:ROC} shows the ROC curves in a typical test in each of the 6 cases in the Table \ref{tab:settings}. Each case uses its first of the four settings, e.g., the constant drifts dataset has $(t_L,d,\Delta t)= (1,1,0.1)$, the dataset for different potentials has $(t_L,d,\Delta t)= (20,1,0.1)$, and the OU processes dataset in Case (c) has $(t_L,d,\Delta t)= (2,1,0.1)$. The datasets for the interacting particles in Cases (d) and (f) have $(t_L,d,\Delta t)= (2,6,0.1)$ and $(t_L,d,\Delta t)= (4,24,0.4)$,  respectively.    

For univariate time series in the Cases (a,b,e), the three algorithms either reach or are close to the optimality benchmark by the LRT. They achieve the optimal benchmark of ``LRT hidden truth'' for the Brownian motion with constant drifts. They are nearly optimal with curves in-between the two LRT benchmarks in distinguishing the diffusions with different potentials and the diffusions with linear or nonlinear drifts.  

For the univariate time series in Case (c) and the multivariate time series in Case (f), the three algorithms are suboptimal as their ROC curves are below the ``LRT numerical'' with $\Delta t = 0.1$ and $\Delta t = 0.4$, respectively. 

However, the three algorithms have unsuccessful classifications in Case (d), which is the multivariate interacting particles with $(t_L,d,\Delta t)= (2,6,0.1)$. Their ROC curves are around the diagonal line. In contrast, the benchmark of ``LRT numerical'' with $\Delta t = 0.1$ has a reasonable ROC curve and the ROC curve of ``LRT hidden truth'' is much higher.  Thus, the data has rich information for the classification, and 
there is room for improvements in these algorithms. We note that the LRT makes use of the model information while the three algorithms are model agnostic. Hence, the success of the ``LRT numerical'' shows the importance of model information in the classification of nonlinear multivariate time series.

In particular, the contrast between the failure in Case (c) and the success in Case (f) invites further examination of the factors that affect the performance of the algorithms. Note that both Case (d) and Case (f) are for the interacting particle systems, and they are different only at $(t_L,d,\Delta t)= (2,6,0.1)$ and $(t_L,d,\Delta t)= (4,24,0.4)$. Thus, in the next section, we examine the algorithms with varying $(t_L,d,\Delta t)$. We will also examine the  dependence of the classification accuracy on randomness and training sample size (Figure \ref{fig:IPS_tL_M}). 
Additionally, a single test is insufficient to draw a conclusive comparison because of the randomness in the data; hence, we run multiple tests in each setting and report the statistics of AUC and ACC in the next section to benchmark the optimality.  

Also, one may notice that the random forest lags behind the other two in Case (c) and the ROCKET lags behind in Case (f), both with rate differences larger than two theoretical standard deviations (0.04). Such differences are due to the randomness in the data in this single test, the statistics from multiple tests in the next section show that no method is superior in all settings.


\begin{figure} [htb] \vspace{-3mm}
\centering
\includegraphics[width =.98\textwidth]{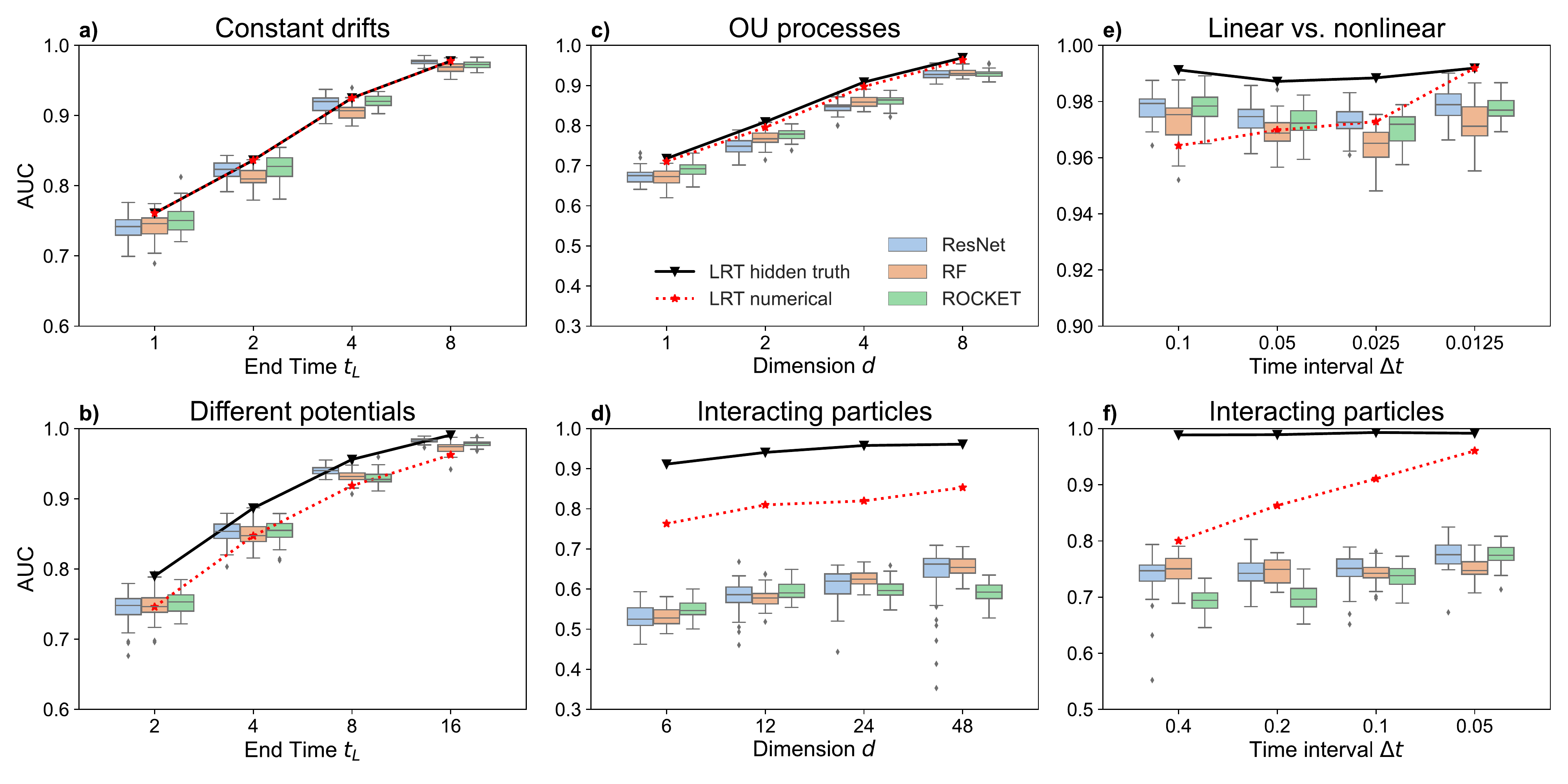} \ifarXiv  \vspace{-4mm} \fi
\caption{AUC for the 6 Cases with varying $(t_L,d,\Delta t)$ in Table \ref{tab:settings}. 
All three algorithms perform similarly: they reach the optimal LRT for Gaussian processes in Case (a), and they are near-optimal in Cases (b,e), 
 suboptimal in Cases (c,f), and are unsuccessful in Case (d). 
 } 
\label{fig:AUC}
\end{figure}
\begin{figure} [h!]  \vspace{-3mm}
\centering
\includegraphics[width =.96\textwidth]{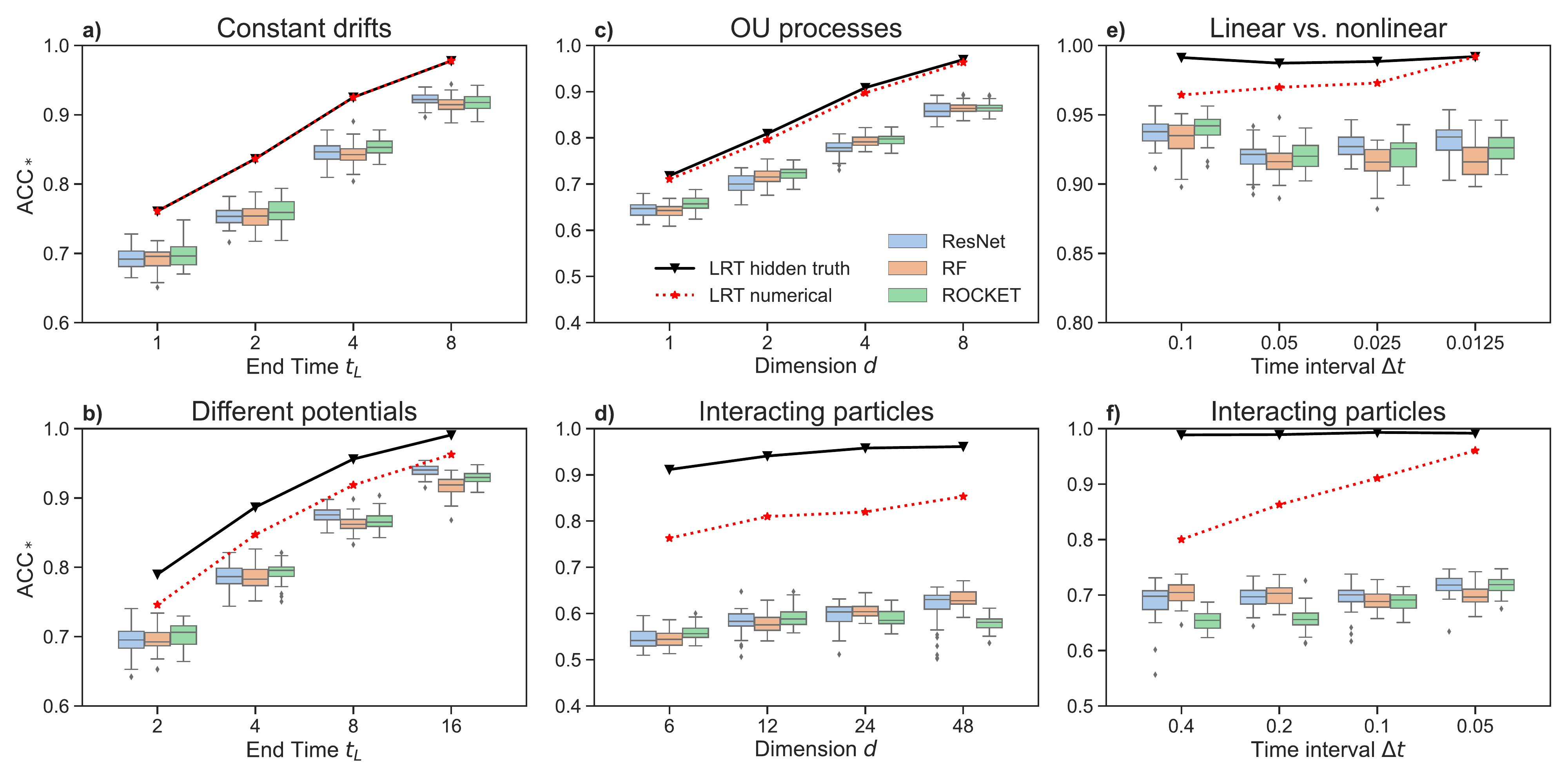}    \ifarXiv  \vspace{-4mm} \fi
\caption{Maximal accuracy ($ACC_{*}$) with varying $(t_L,d,\Delta t)$ in Table \ref{tab:settings}. All three algorithms are suboptimal in comparison with the LRT benchmarks. 
 }  \vspace{-3mm}
\label{fig:ACC}
\end{figure}

\subsection{Optimality benchmarking in AUC and maximal accuracy}
We benchmark the optimality of a classifier by examining the statistics of the AUC and optimal accuracy (ACC$_*$) in 40 independent simulations for every 4 settings of the 6 cases in Table \ref{tab:settings}. We present the box-and-whisker plots (the minimum, the maximum, the sample median, the first and third quartiles, and the outliers) of the AUC and ACC$_*$, which reflect the randomness in the classifications. 

Recall that the ``LRT hidden truth'' provides an upper bound of optimality and the ``LRT numerical'' provides a low baseline for them. Thus, a classifier achieves the optimality for the Gaussian processes if its AUC and ACC$_*$ concentrate around the ``LRT hidden truth''. A classifier is \emph{suboptimal} if its AUC or ACC$_*$ is below the baseline of ``LRT numerical'', particularly when the temporal sampling frequency of observation is relatively low.  We say it is \emph{near optimal} when its statistics lie in between the benchmark lines, particularly when the two lines are close.

Figure \ref{fig:AUC} shows the statistics of the AUCs in the six cases with varying path time lengths $t_L$, dimension $d$ and temporal sampling frequency (through $\Delta t$). In the case of univariate time series data, the three algorithms achieve the optimality represented by the LRT hidden truth for the Gaussian process in Case (a), and they are near optimal for nonlinear time series in Cases (b,e). They are unsuccessful in all settings in Case (d), the high-dimensional interacting particle system with short sample paths, and they are suboptimal in Cases (c,f). These results agree with those from the ROC curves. 

Additionally, we notice two patterns. (i) The AUC increases as the path length in time $t_L$ or the dimension $d$ increases, which can be clearly seen in Cases (a,b,e,d).  (ii)  The AUC of the three methods is not sensitive to the temporal sampling frequency of observation, because Cases (e,f) show that the AUC changes insignificantly as $\Delta t$ refines.  Note that the slopes of the LRT benchmarks in Case (c) are much steeper than those in Case (d). This is because the entries of the OU processes are independent, whereas the entries of the interacting particles are correlated through the interactions. Thus, the increment of AUC is due to the increased effective sample size through either $d$ or $t_L$. Such patterns of AUC's dependence on path length and sample size will be further examined in Figure \ref{fig:IPS_tL_M} for the interacting particle systems.   

 Figure \ref{fig:ACC} shows the statistics of the maximal accuracy ($ACC_*$) in the six cases. It turns out that all three algorithms have smaller maximal accuracy than the benchmark of ``LRT numerical'' (not to mention the ``LRT hidden truth''). Thus, there is room for their improvement.  On the other hand, the two patterns on the dependence of $(t_L,d,\Delta)$ are similar to those observed in AUC in Figure \ref{fig:AUC}.

Figure \ref{fig:Duration} shows the statistics of the computation time in training of these tests. The computation is carried out on a node of 
3.0GHz Intel Cascade Lake 6248R with 48cores, 192GB RAM  1TB NMVe local SSD. 
The figure shows that the random forest (RF) has a controlled computation time for all cases. The computation time of either ResNet or the ROCKET increases in the path length ($L= \frac{t_L}{\Delta t}$) as shown in Cases (a,b,e,f), and is not sensitive to the dimension $d$ as Cases (c,d) suggests. The ResNet has the largest computation time in most cases. 
The LRT benchmarks are not shown here because their computation time is negligible (since they only involve the evaluations of the likelihood ratio). 


\begin{figure} [thb] \vspace{-3mm}
\centering 
\includegraphics[width =.96\textwidth]{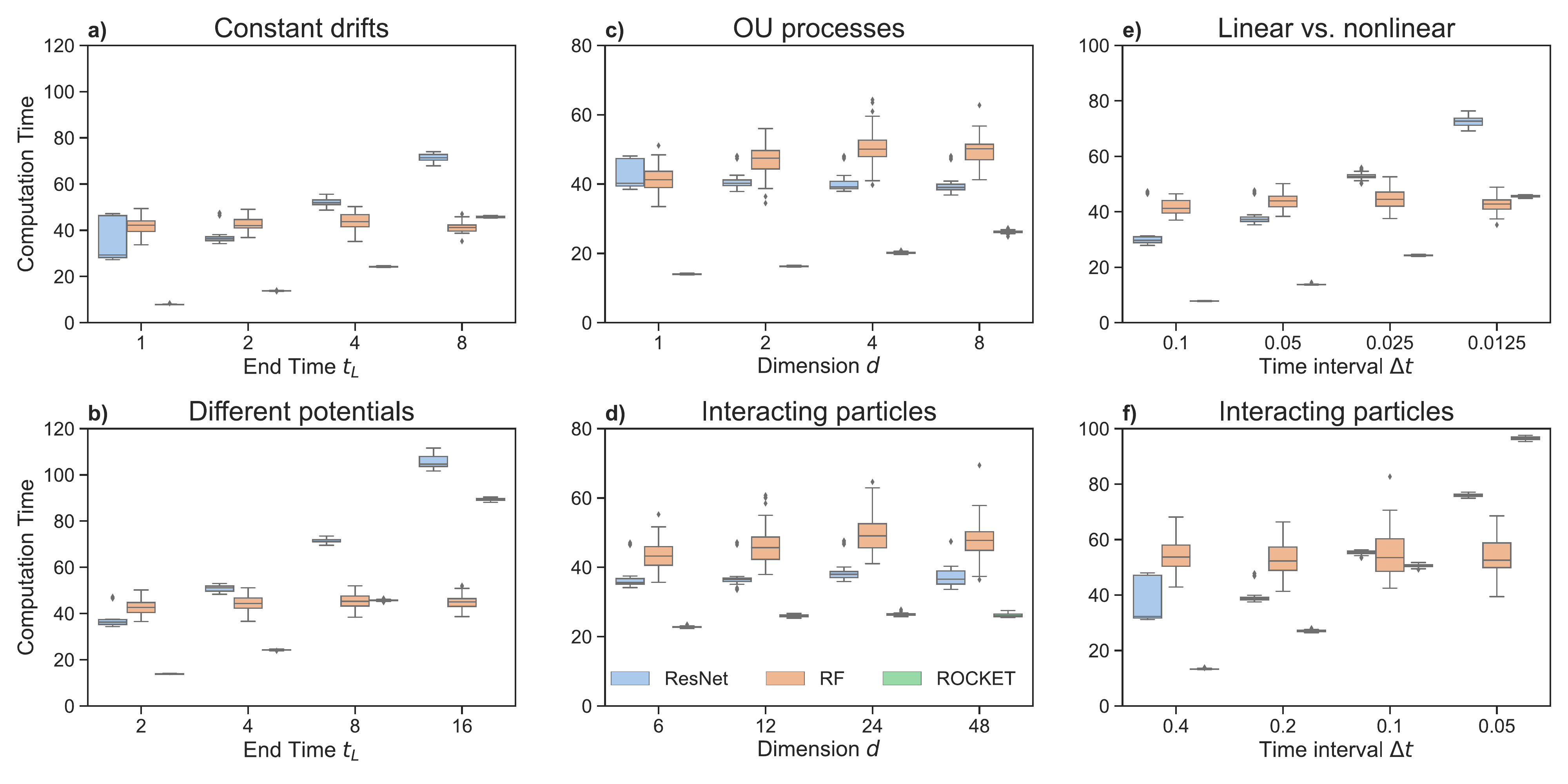}  \ifarXiv  \vspace{-3mm} \fi
\caption{Computation time (in seconds) for the tests with varying $(t_L,d,\Delta t)$ in Table \ref{tab:settings}. The random forest (RF) has a controlled computation time. The ResNet and the ROCKET have computation times increasing with the path length ($L= \frac{t_L}{\Delta t}$) in Cases (a,b,e,f), and not sensitive to the dimension $d$ in Cases (c,d). The ROCKET has the smallest computation time when the  length $L$ is not large. 
}  \vspace{-2mm}
\label{fig:Duration}
\end{figure}

Figure \ref{fig:IPS_tL_M} further examines the dependence of the classification performance on the path length $t_L$, the randomness (in terms of $\sigma$), and the training sample size in Cases \textbf{a)}-- \textbf{c)}, respectively, for the interacting particle systems. 
 These cases show that the AUC of each method increases when either the path time length increases, or the randomness decreases, or the training sample size increases. In particular, Case \textbf{c)} shows that a growing training sample size can significantly improve the AUC of each algorithm; yet, with a training sample size of 4000, their AUCs are far below the LRT benchmarks (which do not need to be trained by taking into account the model information). Additionally, we note that the variation of each algorithm reduces as the sample size increases, indicating that the learning error decays in the sample size. The ResNet has the largest variation among the three algorithms, but its performance improves the most when the sample size increases.

\begin{figure} [thb] \vspace{-1mm}
\centering 
\includegraphics[width =.96\textwidth]{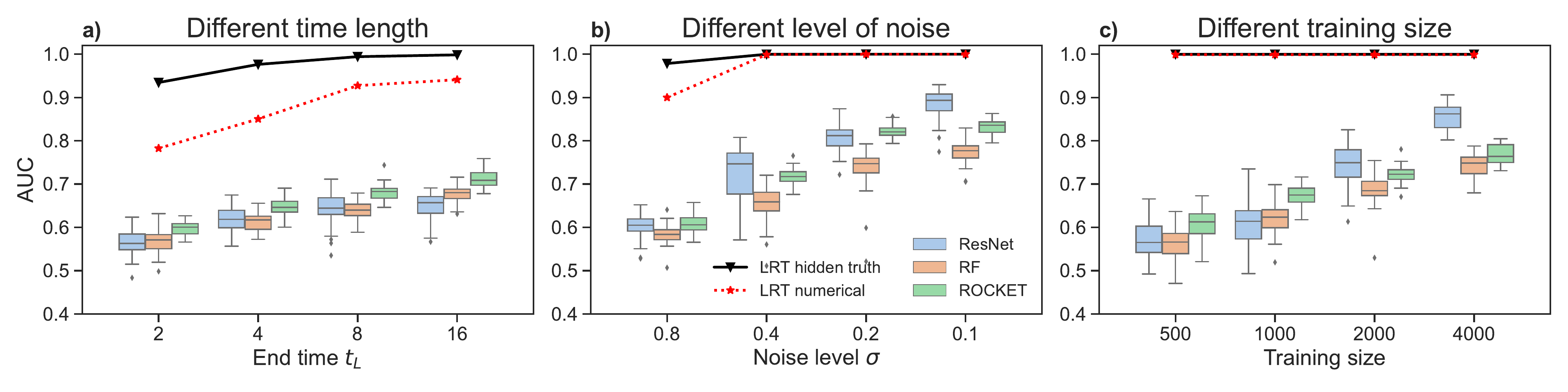}  \ifarXiv  \vspace{-3mm} \fi
\caption{AUC for interacting particles in three additional settings: \textbf{a)}: $t_L\in\{2,4,8,16\}$, $\sigma= 1$ and $M=2000$; \textbf{b)}: $\sigma\in\{0.8,0.4,0.2,0.1\}$, $t_L=2$ and $M=2000$; \textbf{c)}: $M_{\text{training}}\in\{500,1000,2000,4000\}$, $(t_L,\sigma)= (2,0.4)$.  In all cases, the test sample size is $500$ and $(d, \Delta t) =(12,0.1)$. 
}  \vspace{-2mm}
\label{fig:IPS_tL_M}
\end{figure}

\vspace{2mm}
In summary, the LRT benchmarks show that all three algorithms can achieve the LRT optimal AUC for univariate time series and multivariate Gaussian processes. However, these model-agnostic algorithms are suboptimal in classifying nonlinear multivariate time series from high-dimensional stochastic interacting particle systems. Also, the maximal accuracy of each algorithm is below the LRT benchmark in all cases, suggesting room for improvement. Importantly, the LRT benchmarks focus on the


\subsection{Discussion} 
The performance of a classifier depends on multiple factors, including the design of the classifier, the training data size, and the properties of the time series (such as its dimension, randomness, time length, and temporal sampling frequency). The LRT benchmarks help separate these factors so that we can better examine the classifier. 
\begin{itemize}
\item The optimal classification accuracy is determined by the distribution of the underlying discrete-time stochastic process from which the time series is sampled. This distribution varies in the properties of the time series, such as its dimension, randomness, time length, and temporal sampling frequency. The optimal classification accuracy increases when the dimension or the time length increases or the randomness reduces, but it is not sensitive to the temporal sampling frequency. Thus, in data collection in practice, it is more helpful to collect data for a longer time rather than a higher sampling frequency.

\item The performance of a classifier is bounded above by the optimal classification accuracy, and it is limited by its structure and the training data size. In particular, the training data size can significantly affect the optimal accuracy. The size needed to achieve a prescribed level of accuracy increases with the uncertainty in the distribution of the time series and the structure of the classifier. A classifier with a larger complexity requires more data to train. The ResNet, which uses neural networks, improves the most from an enlarging sample size compared to the random forest and ROCKET, which use simpler designs. We expect a bias-variance trade-off for which one can select the degree of complexity of the algorithms adaptive to data size, and we leave this as future work.  

\item The model-agnostic TSC algorithms do not use the model information and rely on data to learn the classifier function; thus, they require a large amount of training data. In contrast, the LRT relies on the model information and does not need to be trained. Therefore, a TSC algorithm using the model information can significantly increase the performance while reducing the training data size. 
 \end{itemize}

\section{Conclusion}
We have shown that the likelihood ratio test (LRT) distinguishing diffusion processes provides ideal optimality benchmarks for time series classification (TSC) algorithms. The benchmarking is computationally scalable and is flexible in design for generating linear or nonlinear time series to reflect the specific characteristics of real-world applications. 

 Numerical tests show that three widely-used TSC algorithms, random forest, ResNet, and ROCKET, can achieve the optimal benchmark for univariate time series and multivariate Gaussian processes. However, these model-agnostic methods are suboptimal compared to the model-aware LRT in classifying high-dimensional nonlinear non-Gaussian processes. 

The LRT benchmarks also show that the classification accuracy increases with either the time length or the time series dimension. However, the classification accuracy is less sensitive to the frequency of the observations. Thus, in data collection, it is more helpful to collect data for a longer time rather than a higher sampling frequency. 
 
 In future work, we propose to quantitatively analyze the dependence on these factors in terms of the effective sample size, the bias-variance trade-off in the training of the algorithms, and the incorporation of model information into the algorithms.

\ifarXiv
\appendix 
\fi  
\section{Appendix}
\subsection{It\^o-diffusion and the Girsanov theorem}\label{sec:diffusion_Girsanov}

\begin{theorem}[Girsanov Theorem]\label{Girsanov}
Let $P_{\theta_i}$ be the probability measure induced by the solution of the SDEs \eqref{eq:sde} for $t \in [t_0, T]$, and let $P_{0}$ be the law of the respective drift-less process. Suppose that the drifts  $\{b_{\theta_i}\}$  and the diffusion $\Sigma=\sigma \sigma'$ fulfill the Novikov condition
$$\mathbb{E}_{P_{\theta_i}}\bigg[\mathrm{exp} \bigg (\frac{1}{2} \int_{t_0}^T b_{\theta_i}(X_{t}, t)^\top \Sigma^{-1} b_{\theta_i}(X_{t}, t)  dt \bigg)\bigg] < \infty. $$
Then, $P_{\theta_i}$ and $P_0$ are equivalent measures with Radon-Nikodym derivative given by 
$$\frac{dP_{\theta_i}}{dP_{0}} \big(X_{[t_0, s]}\big)=\mathrm{exp}\bigg( - \int_{t_0}^{s} b_{\theta_i}^\top \Sigma^{-1} dX_{t}+ \frac{1}{2} \int_{t_0}^{s} \left[ b_{\theta_i}^\top\Sigma^{-1} b_{\theta_i}\right] (X_t)  d t \bigg)$$
for all $s \in [t_0, t]$ and $X_{[t_0,s]}=(X_{t})_{t \in [t_0, s]}$. In particular, the likelihood ratio between $P_{\theta_1}$ and $P_{\theta_0}$ is 
$$\frac{dP_{\theta_1}}{dP_{\theta_0} } \big(X_{[t_0, s]}\big)=\mathrm{exp}\bigg( - \int_{t_0}^{s} [b_{\theta_1}-b_{\theta_0}]^\top \Sigma^{-1} dX_{t}+ \frac{1}{2} \int_{t_0}^{s} \left[ b_{\theta_1}^\top\Sigma^{-1} b_{\theta_1} - b_{\theta_0}^\top\Sigma^{-1} b_{\theta_0}\right](X_t)  d t \bigg).$$
\end{theorem}
Theorem \ref{Girsanov} can be found in \cite[Section 3.5]{KS98} or \cite[Section 8.6]{oksendal2013_sde}.

\subsection{Sampling error in the classification rates}\label{sec:proof_sampling_error}
\begin{proof}[Proof of Lemma \ref{lemma:sampling_error_inTest}] Fix a threshold $k$, the classifier defines a random variable $\xi = \xi(\bx)= F(\bx,k)$.  Then, conditional on $\theta_i$ with $i\in \{0,1\}$, the random variable $\xi$ has a Bernoulli distribution that takes the value $1$ with a probability $\alpha_k^i$. In particular, the test samples $\{\bx_j\}_{j=1}^m$ lead to samples $\{\xi_j\}_{j=1}^m$ of $\xi$, and the empirical approximations of the FNR and TNR by these samples are 
\[
\widehat{\alpha}_{k,m}^{i}  =\frac{1}{m}\sum _{j=1}^m \xi_j , \text{ conditional on } \theta_i, \, i=0,1.
\]
Therefore, by the Central Limit Theorem, the empirical estimators converge in distribution
\[
\sqrt{m} [\widehat{\alpha}_{k,m}^{i}  - \alpha_{k}^{i} ] \to \mathcal{N}(0, \sigma_{\xi,i}^2), \text{ where } \sigma_{\xi,i}^2 = \alpha_{k}^{i}(1-\alpha_{k}^{i}) 
\]
as $m\to \infty$ for each  $i=0,1$. 
Also, the Hoeffding's inequality (see e.g., \cite{CuckerSamle02,CZ07book,gine2015mathematical}) implies that for any $\epsilon>0$, 
\[
\prob( \lvert \widehat{\alpha}_{k,m}^{i}  - \alpha_{k}^{i} \rvert  >\epsilon ) \leq 2 e^{-\frac{m\epsilon^2}{2}}, 
\]
which provides a non-asymptotic bound for each $m> 0$. 
\end{proof}
\subsection{Hypothesis testing and the Neyman-Pearson lemma}\label{sec:hypothesis_testing}
Here we briefly review the hypothesis testing inference method in statistics \cite[Chapter 8]{CasellaBerger01}. Recall that a hypothesis test is a rule that specifies for which sample values the decision is made to accept a hypothesis $H_0$ as true, and reject the complement hypothesis $H_1$. We assume that the family of distributions of the samples are parametrized by $\theta \in \Theta$, where $\Theta$ is the entire parameter space. We denote that the null alternative hypotheses by $ H_0: \theta \in \Theta_0 $ and $ H_1: \theta \in \Theta_0^c$, respectively, where $ \Theta_0 $ is a subset $\Theta$. The binary classification is therefore a hypothesis testing with $\Theta=\{\theta_0,\theta_1\}$ and $\Theta_0=\{\theta_0\}$.

The likelihood ratio test is as widely applicable as maximum likelihood estimation. When there are two parameters, it is defined as follows. 
\begin{definition} [Likelihood Ratio Test.] Let the probability density function (or probability mass function) corresponding to $\theta_i$ be $f(x\mid \theta_i)$ for $ i=0,1 $. The likelihood ratio statistic for testing $ H_0: \theta= \theta_0 $ versus $ H_1 : \theta= \theta_1$ is:
\begin{equation*}
	\lambda(x) = \frac{ f(x \mid \theta_1) }{f(x \mid \theta_0)}. 
\end{equation*}
A likelihood ratio test (LRT) is any test that determines the rejection region for $H_0$ by $\lambda(x)$.  
\end{definition}

The LRT in  \eqref{eq:classifier_fn_LRT} determines the rejection region using the log-likelihood $l(x) = \log \lambda(x)$. The rejection region with  threshold $k\in (0,1)$ is equivalent to 
	\begin{equation*}
R_k^{{\scriptscriptstyle{LRT}}} = \{\bx:  \frac{1}{\lambda(x)+1} >k\} =  \{\bx:  \lambda(x)  > \frac{k}{1-k} \}. 
\end{equation*} 

The reject region is selected to control the probability of falsely rejecting $H_0$, i.e., false negative rate (FNR).  
Meanwhile, it is also desirable to control the false positive rate (FPR), e.g., reduce the possibility of false alarms.  

The hypothesis tests are evaluated by the probabilities of making mistakes. A strategy to compare hypothesis tests is to control the FNR in a class and compare the FPR. The power function provides a tool to define the class. 
\begin{definition}[Power function, size $\alpha$ test.]  
The \emph{power function} of the hypothesis test with a rejection region $R$ and sample $x$ is the probability
$ \beta(\theta) = \prob(x \in R \mid \theta) $ 
as a function of $\theta\in \Theta$. A test with power function $\beta$ is a \emph{size $\alpha$ test} if $ \sup_{\Theta_0} \beta(\theta) = \alpha $; a test with power function $\beta$ is a \emph{level $\alpha$ test} if $ \sup_{\Theta_0} \beta(\theta) \leq \alpha $.
\end{definition}

An ideal hypothesis test would have a power function $\beta(\theta) = 0$ for all $\theta\in \Theta_0$ and $\beta(\theta) =1$ for all $\theta\in \Theta_0^c$. Thus, a good test would have  $\beta(\theta)$ close to $0$ for all $\theta\in \Theta_0$ and $\beta(\theta)$ near $1$ for all $\theta\in \Theta_0^c$. 

Next, we define the uniformly most powerful test as the test with the smallest FPR uniformly for all $\theta\in \Theta_0^c$ in the class of tests with a controlled FNR.  
\begin{definition} [ Uniformly Most Powerful (UMP) Test]
Let $\mathcal{C}$ be a class of tests for testing $ H_0: \theta \in \Theta_0 $ versus $ H_1: \theta \in \Theta_0^c $. A test in class $\mathcal{C}$, with power function $\beta(\theta)$, is a uniformly most powerful (UMP) class $\mathcal{C}$ test if $ \beta(\theta) \geq \beta'(\theta) $ for every $ \theta \in \Theta_0^c $ and every function $ \beta'(\theta) $ that is a power function of a test in class $\mathcal{C}$. 
\end{definition}

The Neyman-Pearson lemma shows that a LRT with a rejection region $R=\{x: \frac{f(x\mid \theta_1)}{f(x\mid \theta_0)}>c \}$ is a UMP test when $\Theta_0=\{\theta_0\}$ and $\Theta_0^c= \{\theta_1\}$ for any $c\in (0,\infty)$ such that $\prob(\{x: \frac{f(x\mid \theta_1)}{f(x\mid \theta_0)}=c\})=0$.   
\begin{theorem} [Neyman-Pearson Lemma]
Consider testing $ H_0: \theta = \theta_0 $ versus $ H_1 : \theta = \theta_1 $, where the probability density function (or probability mass function) corresponding to $\theta_i$ is $f(x \mid \theta_i)$ for $ i=0,1 $, using a test with rejection region R that satisfies 
\begin{equation} \label{NPLemma_condition1}
	\left\lbrace \begin{aligned}
		x \in R, \text{ if } f(x \mid \theta_1) >c  f(x \mid \theta_0)
		\\
		x \in R^c, \text{ if } f(x \mid \theta_1) < c  f(x \mid \theta_0)
	\end{aligned} \right.
\end{equation}
for some $c > 0$, and
\begin{equation} \label{NPLemma_condition2}
	\alpha = P_{\theta_0 }(X \in   R)
\end{equation} 	
Then:
\begin{enumerate}
	\item (Sufficiency) Any test that satisfies 
	{\rm(\ref{NPLemma_condition1})} and {\rm(\ref{NPLemma_condition2})} is a UMP level $\alpha$ test. 
	\item (Necessity) If there exists a test satisfying {\rm(\ref{NPLemma_condition1})} and {\rm(\ref{NPLemma_condition2})} with $c>0$, then every UMP level $\alpha$ test is a size $\alpha$ test {\rm(satisfies (\ref{NPLemma_condition2}))} and every UMP level $\alpha$ test satisfies {\rm(\ref{NPLemma_condition1})} except perhaps on a set A satisfying $P_{\theta_0 }(X \in A)$ = $P_{\theta_1 }(X \in A) = 0$. 
\end{enumerate} 
\end{theorem}

\ifjournal   \end{appendices}   \fi

\ifarXiv
 \paragraph{Acknowledgments}   Terry Lyons was funded in part by the EPSRC [grant number EP/S026347/1], in part by The Alan Turing Institute under the EPSRC grant EP/N510129/1, the Data Centric Engineering Programme (under the Lloyd’s Register Foundation grant G0095), the Defence and Security Programme (funded by the UK Government) and the Office for National Statistics \& The Alan Turing Institute (strategic partnership) and in part by the Hong Kong Innovation and Technology Commission (InnoHK Project CIMDA). F.L. and Y.K. are partially supported by the United States Department of Energy grant DE-SC0021361. The work of F.L. is partially funded by the Johns Hopkins University Catalyst Award and the United States Air Force grant FA95502010288. The computation is carried out on the clusters of the Maryland Advanced Research Computing Center. FL would like to thank Professors Geoff Webb and Xingjie Li for their helpful comments on the paper. 

  \fi

\ifarXiv
\bibliographystyle{myplain}
\fi
{\small
\bibliography{ref_sigM,ref_FeiLU2023_2,ref_tsc}
}

\end{document}